\pdfoutput=1
\documentclass[twoside,11pt]{article}

\usepackage{amsmath}
\usepackage{tikz}
\usepackage{float}
\usepackage{longtable}
\usepackage{array}
\usepackage{needspace}
\usetikzlibrary{arrows.meta,positioning,shapes.geometric}

\usepackage[preprint]{jmlr2e}
\hypersetup{
  pdftitle={Human--LLM Deliberation as Interactive Proof: Conditions for Verifiability Without Transparency},
  pdfauthor={Baotong Zhang; Dean Foster; Jo\~ao Sedoc},
  pdfsubject={Conditional guarantees for resource-bounded human--LLM verification},
  pdfkeywords={Interactive proofs, Human-AI decision-making, Resource-bounded verification, Verifiability, Human-LLM deliberation}
}

\renewcommand{\Pr}{\operatorname{Pr}}

\newtheorem{condition}{Condition}
\newtheorem{assumption}{Assumption}

\newif\ifextended
\extendedtrue

\newcommand{\compactprotocol}{
  \setlength{\itemsep}{0.15em}
  \setlength{\parsep}{0pt}
  \setlength{\parskip}{0pt}
  \setlength{\topsep}{0.25em}
  \setlength{\partopsep}{0pt}
}

\ShortHeadings{Verifiability Without Transparency}{Zhang, Foster, and Sedoc}
\firstpageno{1}

\begin{document}
\raggedbottom

\title{Human--LLM Deliberation as Interactive Proof: Conditions for Verifiability Without Transparency}

\author{\name Baotong Zhang \email baotong.zhang@stern.nyu.edu \\
       \addr Stern School of Business\\
       New York University\\
       New York, NY, USA
       \AND
       \name Dean Foster \email
       foster@UPenn.edu \\
       \addr Amazon \\
       New York, NY, USA 
       \AND
       \name João Sedoc \email jsedoc@stern.nyu.edu \\
       \addr Stern School of Business\\
       New York University\\
       New York, NY, USA}

\maketitle
\thispagestyle{plain}

\begin{abstract}
When an LLM supplies an argument a user could not produce, how can the user decide whether to accept its claim? Inspired by interactive proofs, we pair a computationally unbounded LLM prover with a resource-bounded human verifier. The verifier checks requested details without seeing the LLM's internals; passed checks accumulate evidence toward an acceptance threshold. Given task-supplied bounds on false passes and human checking errors, we prove anytime-valid soundness against adaptive provers: the probability of ever accepting a false claim is at most a chosen error level. For a true claim, acceptance within the budget has probability at least the product of per-round probabilities that the LLM responds adequately and the human checks correctly, provided the threshold is reachable in time. Each additional check guards against false claims but gives a true claim another chance to fail. Which review lengths meet both error targets depends on the verifier's effort, cognitive load, expertise, and fatigue. With equal resources, we identify conditions under which the bounds certify a review by local checks but not a one-shot check. In LLM-only simulations on GPQA Diamond and Humanity's Last Exam, whether answers accepted without revision were more accurate varied by task and model pair.
\end{abstract}


\section{Introduction}
\label{sec:intro}

LLMs increasingly give users arguments they could not produce themselves, from multistep derivations to audits of long schedules. Reading such an argument does not by itself tell the user whether its claim should be accepted. Our central question is whether the user can obtain grounds for that decision through challenges within their own resource budget. Interactive proofs motivate a division of work between a powerful prover and a bounded verifier \citep{Goldwasser1989TheKC}: the prover supplies arguments or evidence, while the verifier selects and executes checks. We study the conditions under which this division supports reliable human--LLM deliberation.

The question arises in human--AI decision support, where predictions can inform consequential choices \citep{kleinberg2018human} and users receive performance information, explanations, or other assistance \citep{lu2025can,lai2019human,lai2023towards}. LLMs extend this setting by producing arguments and responding to follow-up questions. Early work on GPT-4 documented broad capabilities across difficult tasks \citep{bubeck2023sparks}, while hallucination research identifies fabrication, contradiction, and unfaithfulness to source material or context \citep{ji2023survey,huang2023hallucination}. Prompt-only self-correction based on a model's own feedback has not been reliably demonstrated on general tasks \citep{kamoi2024llm}.

A user faced with an uncertain answer may ask the LLM to explain it. The explanation can make the output more inspectable \citep{zhao2024explainability}, but it need not describe how the answer was actually produced. Plausibility and faithfulness are distinct properties \citep{jacovi-goldberg-2020-towards}: self-explanations may not reflect the process that generated the answer \citep{agarwal2024faithfulness}, and chain-of-thought can omit factors that biased it \citep{turpin2023language}. An explanation therefore cannot be relied on to show how the model reached its answer, and explanations optimized for plausibility may even mislead users \citep{jin2023plausibility}. We therefore ask whether the content of an explanation can be checked. Following the verifiability criterion of \citet{fok2024verifiability}, we focus on evidence about the claim's consequences or a supporting derivation. Here, \emph{verifiability without transparency} means that the user can decide whether to accept a claim on the basis of such checks, without access to the LLM's internal state or a reconstruction of its generation process.

The user need not request the entire supporting argument at once. Starting from a high-level argument, they can question a premise, request support for an inference, propose a counterexample, or ask for an executable test. The LLM supplies the requested details, while the claim under review stays the same throughout. Each requested detail is less to check than the full argument, which may keep the review manageable for a user with limited attention or domain expertise. This division of labor also suits human reasoning: people tend to be better at evaluating arguments put to them than at scrutinizing their own \citep{mercier2011why}. Interactive critique has helped in related settings, including human--AI deliberation for decision-making \citep{ma2025towards} and the refinement of LLM outputs through self-feedback or tool-assisted critique \citep{madaan2023self,gou2024critic}.

For example, an organizer reviewing a conference schedule can challenge the claim that it has no conflicts by checking selected speaker and room requirements against fixed records (Section~\ref{sec:running_example}). For such checks to count as evidence, they must have a known chance of exposing a conflict if one exists, however the LLM answers. Here, randomization provides that chance: if the organizer checks one of the schedule's $J$ requirements chosen uniformly at random, a conflicting schedule fails the check with probability at least $1/J$. Other tasks obtain such guarantees from an exact external checker, such as a proof assistant, or from a randomized test, such as Freivalds' check of a matrix product (Appendix~\ref{app:worked_certificates}). Each application must supply these guarantees; we take them as inputs and determine when the resulting checks justify accepting the claim. Questions that only ask for clarification can guide later checks but do not count as evidence.

We formalize this setting with a computationally unbounded LLM prover and a resource-bounded human verifier. In a classical interactive proof, the verifier is an algorithm that executes every check correctly. Our model allows both parties to fail: the LLM may give an inadequate answer to a legitimate question, and the human may make mistakes when carrying out a check. When every response is adequate and every check is executed correctly, our guarantees reduce to those of the idealized protocol. The distinction matters because a check that is cheap for an algorithm can still be hard for a person. Comparing two schedule entries is easy; verifying a long proof in a single pass may exceed the reviewer's attention, knowledge, or available time. We call a review \emph{$(\delta,\gamma)$-certified} if the supplied bounds guarantee that it accepts a false claim with probability at most $\delta$ and accepts a true claim from an honest LLM with probability at least $1-\gamma$.

We make three contributions:
\begin{enumerate}
\item \textbf{A deliberation protocol with a round budget} (Section~\ref{sec:protocol}). The human verifier repeatedly chooses what to check, possibly adapting each question to the dialogue so far, and the LLM supplies the requested details. Each passed check adds evidence. The claim is accepted once the accumulated evidence crosses a threshold fixed in advance, rejected as soon as a check fails, and rejected by default when the budget of rounds runs out.
\item \textbf{Error guarantees under imperfect responses and checking} (Section~\ref{sec:zkp}). Given per-round bounds on how often a check can pass a false claim and how often the human executes it incorrectly, the probability that the verifier ever accepts a false claim is at most $\delta$, whatever strategy the LLM uses (Theorem~\ref{thm:anytime_soundness}). For a true claim, the probability of acceptance within the budget is at least the product, over rounds, of the chances that the LLM answers adequately and the human checks correctly (Theorem~\ref{thm:conditional_completeness}). Neither result assumes that rounds are independent. Questions may follow the dialogue if each round uses a reference check, such as a uniformly random requirement, with a probability fixed before the round. With a polynomial-time verifier and a large enough gap between the two error bounds, the protocol is an interactive proof in the standard sense (Corollary~\ref{thm:conditional_ip}).
\item \textbf{When limited human resources suffice} (Section~\ref{sec:fallibility}). Requiring more passed checks makes false acceptance less likely, but each additional round gives a true claim another chance to fail and spreads a fixed effort budget more thinly. We characterize the review lengths that meet both certification targets (Proposition~\ref{prop:effective_region_characterization}) and model a person's checking reliability as a function of the effort spent on a check, its cognitive load, and the reviewer's expertise, with the round budget reflecting fatigue. Under this model, a single one-shot check can be too demanding to execute reliably, while a sequence of smaller local checks under the same budget meets both targets (Proposition~\ref{prop:separation}). In a matrix-product example with assumed cognitive loads, five randomized local checks are certified while direct multiplication with the same effort is not. The proposition states how much the load of each local check must fall to offset the cost of decomposition.
\end{enumerate}

Section~\ref{sec:experiments} complements the theory with LLM-only simulations on GPQA Diamond and Humanity's Last Exam, in which a second LLM plays the verifier. Answers accepted without revision were more accurate than the rest for some model pairs but not for others, and the verifier rejected most deliberately wrong defenses in two stress tests. Because no human participates, these simulations do not estimate the human reliability parameters of the theory. Appendix~\ref{app:hvzk} addresses privacy separately: if the full interaction is honest-verifier zero knowledge and the stopped view can be computed efficiently from the full view, stopping at the acceptance threshold preserves this property.
\section{Related Work}
\label{sec:related}

Most work on verifying model outputs either uses verification to train a model or assumes that whoever checks does so correctly. In our setting, a person reviews an answer that a fixed model has already produced and may make mistakes while doing so. Where prior work offers guarantees, we also ask whether they bound false acceptance for every prover strategy or an average error over a distribution of tasks.

\subsection{Interactive Proofs and Prover--Verifier Training}

Interactive proofs motivate our central abstraction: a powerful prover makes a claim, while a weaker verifier interrogates it before deciding whether to accept. Classical theory formalizes this division through completeness and soundness, with the verifier modeled as an algorithm that executes every check exactly \citep{Goldwasser1989TheKC}. Neither party is this reliable in human--LLM review, where a person may misjudge even a prescribed check and an honest LLM may fail to answer a legitimate question adequately. Our earlier work \citep{zhang2025human} formulated human--LLM deliberation as an interactive proof that extends to zero knowledge, but its guarantees assume an idealized Effective Verifier and independent rounds. The present model allows both failures, bounds them separately, and does not require independent rounds. Zero knowledge requires that the interaction reveal nothing beyond the validity of the claim; Appendix~\ref{app:hvzk} gives conditions under which stopping at our acceptance threshold preserves its honest-verifier form.

In machine learning, prover--verifier games have mainly served as a training device, used to make a model's outputs easier for a weaker verifier to check. Prover--Verifier Games \citep{anil2021learning} and neural interactive proofs \citep{hammond2025neural} study when equilibria of such games correspond to valid proof systems, and \citet{kirchner2024prover} find that this training also makes solutions easier for time-constrained human raters to check. Self-Proving Models \citep{amit2025models} are closer to our setting because verification happens at inference time, when a fixed, manually specified verifier interactively checks each answer the trained model produces. Soundness is required against every prover, while the model's ability to convince the verifier is measured over an input distribution. We keep the same separation between false acceptance and the verification of true claims.

\subsection{Scalable Oversight, Debate, and Consultancy}

Scalable oversight asks how a weaker overseer can supervise a stronger AI system \citep{bowman2022measuring}. Much of this work uses oversight to produce training signals. Iterated amplification builds a training signal for a hard question by combining answers to easier subquestions \citep{christiano2018amplification}, and weak-to-strong supervision trains a strong model on labels from a weaker one \citep{burns2024weak}.

At deployment, the overseer is often a person deciding whether to accept one model's answer. This interaction has the structure of \emph{consultancy}, a baseline for debate introduced by \citet{michael2023debate} in which a single model argues for an answer and responds to a weaker judge's questions. In debate, two models instead argue for opposing answers before the judge \citep{irving2018debate}. Experiments with human and LLM judges \citep{michael2023debate,khan2024debating,kenton2024scalable} find that judges are less accurate under consultancy than under debate when the consultant is assigned the wrong answer half the time; in human consultancies, many judge errors arise because the consultant obscures the relevant evidence \citep{michael2023debate}. We keep the single-model setting, since a user typically reviews one model's answer. To limit the consultant's influence, acceptance in our protocol depends only on passed checks, each with a bounded chance of passing a false claim. Because the LLM's answers can also steer which questions come next, this bound must hold after every dialogue history; Section~\ref{sec:challenge_design} shows how to keep it valid when follow-up questions adapt.

Formal analyses of debate \citep{browncohen2024debate} and of single-prover protocols \citep{chen2026avoiddebate} prove guarantees against dishonest provers by giving the human a different role from ours. They write the task as a computation that occasionally asks a human for a judgment, and they treat these judgments as defining what is correct. The verifier is an efficient algorithm, and the protocols ensure that it needs only a few human judgments to catch a wrong claim. In our setting the human is the verifier and may carry out a check incorrectly, so our guarantees are stated in terms of two per-round bounds that each application supplies: how often a check passes a false claim, and how often the human executes it incorrectly.

\subsection{Local Checks and Sequential Evidence}

Our protocol needs two things: a way to break a claim into checks that the verifier can carry out, and a rule for deciding when enough checks have passed. Work on argumentation and fact-checking provides the first. Argumentation frameworks formalize which arguments survive a dialogue, given how they attack one another and what each party has committed to \citep{dung1995acceptability,modgil2014aspic,prakken2005coherence}. FActScore splits a generated text into atomic facts and checks each against a knowledge source \citep{min2023factscore}; ProgramFC decomposes a complex claim into sub-questions answered by separate modules \citep{pan2023programfc}; Draft, Sketch, and Prove turns an informal proof into formal subgoals that a proof assistant checks exactly, although the formal theorem may not match the original informal claim \citep{jiang2023draft}. These methods supply local checks of the kind our verifier issues, but leave open when enough checks have passed to accept a claim.

Sequential testing addresses that question. When checks are chosen adaptively and the reviewer may stop at any time, a fixed-sample test loses its error guarantee; test supermartingales and e-processes retain it under such optional stopping \citep{shafer2011test,ramdas2023gametheoretic}. Our soundness proof uses this argument. Two recent methods apply e-processes to decisions about AI outputs using scores from black-box evaluators. E-valuator monitors an agent's trajectory and controls the rate of false alarms on trajectories that would have succeeded \citep{sadhuka2025evaluator}. \citet{cho2026release} decide when a generate-and-revise workflow should release its output, and bound the probability of releasing on tasks the workflow cannot reliably solve. Both turn evaluator scores into evidence by calibrating them with labeled or reference data, so their guarantees depend on how well those data represent deployment. Our bound on how often a check passes a false claim comes from the structure of the task, such as a uniformly random choice among requirements, and applies to each claim without calibration data. The cost is that each application must justify this bound and must also bound how often the human executes the check incorrectly.

\subsection{Model-Based Verification and Selective Prediction}

Another way to check an LLM's answer is to have a model do the checking. Outcome verifiers rank candidate solutions to math problems \citep{cobbe2021training}, process supervision scores each intermediate step \citep{lightman2023verify}, and Chain-of-Verification has the model pose its own verification questions before revising \citep{dhuliawala2023chain}. These checkers are judged by average performance, which can hide systematic failures. LLM judges agree well with human preferences overall, yet favor answers by position, length, or their own authorship \citep{zheng2023llm}, and prompted self-correction is unreliable without external feedback \citep{kamoi2024llm}. A prover could exploit such biases, so a model-based checker can serve in our protocol only if its chance of passing a false claim is bounded for every prover strategy.

Our simulations use Prover--Verifier Deliberation (PVD) \citep{sedoc2026trust}, an LLM-only setup in which a verifier model challenges a prover model and an answer is selected when it is accepted without revision. Beyond the GPQA Diamond and HLE runs drawn from that work, we add two stress tests in which the prover is instructed to defend a wrong answer, probing the false acceptances that our soundness guarantee bounds. Selecting some answers and abstaining on the rest is the setting of selective prediction, where a rule is judged by its error rate on the cases it answers. \citet{geifman2017selective} bound this rate with high probability over new cases from the same distribution as the labeled data, whereas our soundness guarantee bounds the probability of accepting each false claim, with no averaging over tasks.

\subsection{Human Checkability and Bounded Certification}

A check that is informative in principle can still be hard for a person to carry out. Human-centered studies address this difficulty through the assistance given to the evaluator. Cognitive forcing reduces overreliance on AI more effectively than explanations alone \citep{bucinca2021trust}, and an LLM-based deliberation interface improved appropriate reliance and decision performance in an exploratory graduate-admissions task \citep{ma2025towards}. Assistance can also backfire. In fact-checking, LLM explanations made people more efficient but led them to accept answers that were convincingly wrong \citep{si2024truthfulness}. Economic models of costly communication similarly treat the receiver's ability, attention, and effort as limits on what a message conveys \citep{dewatripont2005modes}.

Splitting a review into smaller checks can make each check easier, but it introduces new sources of error. In factored-evaluation experiments where each judge saw only one step of an argument, \citet{saunders2020arguments} could flag some invalid arguments, yet judgments of individual steps varied widely and the combined verdict was unreliable. Model-written critiques help people find flaws they would otherwise miss \citep{saunders2022self}, so assistance can also lower the difficulty of each check. \citet{fok2024verifiability} argue that AI assistance improves human decisions only when it makes the AI's recommendation easier to verify. Our protocol meets part of this requirement, since an LLM response counts as evidence only through a check that the human executes; whether that check is easy enough to execute reliably is a separate question. Section~\ref{sec:fallibility} takes it up by modeling execution reliability as a function of effort, expertise, and the load of each check, and identifies when a sequence of local checks is certified while a single one-shot check under the same budget is not.

The next section specifies the fixed-claim protocol; Section~\ref{sec:zkp} states the local conditions and derives its soundness and completeness guarantees.
\section{Interactive Human--LLM Deliberation Model}
\label{sec:protocol}
This section defines the interaction studied in the paper. An LLM prover ($P$) and a human verifier ($V$) deliberate over a decidable statement until $V$ accepts or rejects it. Because human attention is limited, the review has a round budget $T$, and a review that reaches no verdict within $T$ rounds ends in \emph{Reject}. Sections~\ref{sec:token_space}--\ref{sec:protocol_def} define the statements, the two parties, and the protocol. Sections~\ref{sec:running_example} and~\ref{sec:diagnostic_example} illustrate them with a scheduling review in which every check is executed correctly.

\subsection{Statements and Dialogue Histories}
\label{sec:token_space}
Let $\Sigma$ be a finite, non-empty alphabet and let $X=\Sigma^*$ be the set of all finite strings over $\Sigma$. Elements of~$X$ are token sequences, used here to represent statements, challenges, and prover responses. A history of messages $z_1,\dots,z_r$ is written $z_1\parallel z_2\parallel\cdots\parallel z_r$, using an encoding from which the individual messages can be recovered. Let $\mathcal H\subseteq X$ denote the set of well-formed encoded histories. If $h\in\mathcal H$ and $z\in X$ is a new message, then $h\parallel z$ denotes the corresponding appended history.

Let $S\subseteq X$ be a set of natural-language statements that can be rendered as logical propositions with decidable truth values. On this domain, fix a total ground-truth map
\[
 v:S\to\{\top,\bot\},
\]
where $\top$ denotes ``true'' and $\bot$ denotes ``false.'' The language of true statements is
\[
 L=\{s\in S: v(s)=\top\}.
\]

Decidability is built into the choice of $S$; statements without a definite truth value are outside the model. The map $v$ fixes what is to be verified, but neither party can query it during the review. Soundness and completeness are stated with respect to $L$.

\subsection{Production, Challenge, and Decision Functions}
\label{sec:agent_functions}
The LLM prover $P$ answers from a private state $\theta\in\Theta$, such as its weights and any retrieved documents, which the verifier does not see. As in classical interactive proofs, the prover is computationally unbounded; only the format and length of its messages are restricted.

\begin{definition}[Production Function]
\label{def:production_function}
The production function of $P$ is a mapping
\[
Q:\mathcal H\times\Theta\to X.
\]
Given a well-formed history $h\in\mathcal H$ and private state $\theta$, $Q(h,\theta)$ produces a prover message $y\in X$.
\end{definition}

For simplicity, $Q$ is deterministic. Stochastic generation can be handled by adding randomness to~$\theta$ or by drawing responses from a conditional distribution instead.

The human verifier $V$ is specified by three objects: the set $C$ of challenges it can ask, the policy $\pi_V$ by which it picks one of them, and the decision function $U$ that determines when the review stops. All three depend on the information the verifier brings to the review. An application must supply procedures for choosing challenges and running checks that fit within the verifier's budget; Section~\ref{sec:zkp} adds polynomial-time requirements only for the correspondence with classical interactive proofs.

\begin{definition}[Verifier Information State]
\label{def:verifier_information}
Let $\mathcal I$ be the set of finite, representable verifier information states. For a statement $s\in S$, the information state $i\in\mathcal I$ collects what $V$ brings to the evaluation of~$s$, such as domain knowledge and available tools. It stays fixed during the review.
\end{definition}

\begin{definition}[Admissible Challenges]
\label{def:challenge_space}
Let $\mathcal C\subseteq X$ be the set of all well-formed challenges. The admissible challenge set is given by a set-valued map
\[
C:\mathcal H\times\mathcal I\to 2^{\mathcal C}\setminus\{\emptyset\}.
\]
Given history $h\in\mathcal H$ and information state $i\in\mathcal I$, $C(h,i)$ is the non-empty set of challenges that $V$ can formulate from $h$ using~$i$. The set may include weak or redundant challenges; Section~\ref{sec:diagnostic_example} specifies which checks count as evidence.
\end{definition}

The set $C(h,i)$ fixes which challenges are available. For example, a challenge can ask about a premise only after the prover has stated it in~$h$, and can request a computation only if $i$ includes a way to check the result. The policy $\pi_V$ specifies how the verifier chooses among these challenges, possibly at random.

Besides the dialogue $h$, the verifier keeps a private record $\upsilon\in X$ that the prover does not see. In the scheduling review of Section~\ref{sec:running_example}, it lists which requirements the organizer has drawn, whether each matched the records, and the evidence accumulated so far. Both the challenge policy and the decision function may depend on this record. It starts from $(s,i)$ and the verifier's randomness and is updated after each check.

\begin{definition}[Challenge Policy]
\label{def:challenge_policy}
The verifier selects a challenge from the conditional distribution $\pi_V(\cdot\mid h,i,\upsilon)$. This distribution assigns probability only to challenges in $C(h,i)$. A deterministic verifier always selects the same challenge at a given history and record.
\end{definition}

\begin{definition}[Decision Function]
\label{def:verification_function}
$V$'s decision function is a possibly randomized rule:
\[
U:\mathcal H\times\mathcal I\times X\to\{\text{Accept},\text{Reject},\text{Challenge}\}.
\]
Given a well-formed history $h$ whose first component is a statement $s\in S$, information state $i$, and current record $\upsilon$, $U(h,i,\upsilon)$ returns one of three outputs:
\begin{itemize}
\item $\text{Accept}$: terminate with acceptance of the claim that $s\in L$;
\item $\text{Reject}$: terminate without accepting the claim;
\item $\text{Challenge}$: request another challenge--response round, if the round budget $T$ has not run out.
\end{itemize}
\end{definition}

Separating challenge selection from the decision rule lets the verifier adapt its questions while the acceptance criterion stays fixed; Section~\ref{sec:zkp} specifies that criterion. To simplify notation, we write $\pi_V(\cdot\mid h,i)$ and $U(h,i)$ from here on and leave the record $\upsilon$ implicit.

\subsection{The Interactive Human--LLM Deliberation Protocol}
\label{sec:protocol_def}

The protocol begins with the LLM's initial response and a first challenge that the human always asks. After each response, the human runs the prescribed check and then accepts, rejects, or asks another question. Review continues until a verdict is reached or the round budget runs out (Figure~\ref{fig:protocol_diagram}).

\begingroup
\setlength{\abovedisplayskip}{4pt}
\setlength{\belowdisplayskip}{4pt}
\setlength{\abovedisplayshortskip}{2pt}
\setlength{\belowdisplayshortskip}{4pt}
\begin{definition}[Deliberation protocol with a round budget]
\label{def:deliberation_protocol}
A deliberation concerning a statement $s\in S$ is a token exchange between $P$ (private state $\theta\in\Theta$) and $V$ (information state $i\in\mathcal I$) in which $P$ aims to convince $V$ that $s\in L$. The statement $s$, together with any records it refers to (such as the schedule in Section~\ref{sec:running_example}), stays fixed during the review; revising them starts a new review. The human can carry out at most $T\ge1$ challenge--response rounds; this round budget reflects limited attention and fatigue. If $V$ has not reached Accept or Reject after $T$ rounds, the protocol outputs Reject. Clarification rounds, whose answers do not count toward acceptance, also consume this budget. Every message must be well formed and within a length limit fixed in advance; a prover message that violates this ends the review with Reject.

\begin{enumerate}
\compactprotocol
\begin{samepage}
\item \textbf{Initialization.}
\begin{itemize}
\compactprotocol
\item Initialize the verifier's record $\upsilon$ from $(s,i)$ and its own randomness, before receiving any prover message.
\item The prover produces the initial response
\[
y_s=Q(\langle s\rangle,\theta),
\]
where $\langle s\rangle$ denotes the one-message history containing the statement.
\item Set
\[
h_0\leftarrow s\parallel y_s,
\qquad m\leftarrow 0.
\]
\end{itemize}
\end{samepage}

\item \textbf{Mandatory First Challenge--Response.}
\begin{itemize}
\compactprotocol
    \item The verifier samples or selects $x_1\sim\pi_V(\cdot\mid h_0,i)$, so $x_1\in C(h_0,i)$. The prover responds with
    \[
    y_1=Q(h_0\parallel x_1,\theta).
    \]
    \item Set
    \[
    h_1\leftarrow h_0\parallel x_1\parallel y_1,
    \qquad m\leftarrow 1.
    \]
\end{itemize}

\item \textbf{Deliberation Loop.} Repeat:
\begin{itemize}
\compactprotocol
    \item \textit{Evaluation.}
    \begin{itemize}
        \item $V$ executes the prescribed check and updates its record, then computes $\omega=U(h_m,i)$ using that record. Here $h_m$ is the history after $m$ completed rounds.
    \end{itemize}
    \item \textit{Termination.}
        \begin{itemize}
\compactprotocol
            \item If $\omega=\text{Accept}$: terminate with verdict $\Omega=\text{Accept}$.
            \item If $\omega=\text{Reject}$: terminate with verdict $\Omega=\text{Reject}$.
            \item If $\omega=\text{Challenge}$ and $m=T$: terminate with verdict $\Omega=\text{Reject}$ (the round budget has run out).
        \end{itemize}
    \item \textit{Challenge--Response} (if $\omega=\text{Challenge}$ and $m<T$):
        \begin{itemize}
\compactprotocol
             \item The verifier samples or selects $x_{m+1}\sim\pi_V(\cdot\mid h_m,i)$, so $x_{m+1}\in C(h_m,i)$. The prover responds with
             \[
             y_{m+1}=Q(h_m\parallel x_{m+1},\theta).
             \]
             \item Set
             \[
             h_{m+1}\leftarrow h_m\parallel x_{m+1}\parallel y_{m+1},
             \qquad m\leftarrow m+1.
             \]
        \end{itemize}
\end{itemize}

\item \textbf{Output.} Let $m'\in\{1,\dots,T\}$ be the number of completed rounds. The protocol outputs the verdict $\Omega\in\{\text{Accept},\text{Reject}\}$ together with the transcript
\[
d^{(m')}=h_{m'}=s\parallel y_s\parallel x_1\parallel y_1\parallel\cdots\parallel x_{m'}\parallel y_{m'}.
\]
If a malformed or over-length message ends the review early, the output is Reject together with the messages received up to that point.
\end{enumerate}
\end{definition}
\endgroup

\begin{figure}[tbp]
\centering
\begin{tikzpicture}[
  font=\footnotesize,
  >=Latex,
  box/.style={
    draw=black!70,
    fill=black!4,
    rounded corners=2pt,
    text width=3.85cm,
    minimum height=1.5cm,
    inner xsep=4pt,
    inner ysep=5pt,
    align=center
  },
  decision/.style={
    draw=black!70,
    fill=black!2,
    diamond,
    aspect=1.7,
    inner sep=3pt,
    minimum width=2.4cm,
    minimum height=1.3cm,
    align=center
  },
  terminal/.style={
    draw=black!70,
    fill=black!6,
    rounded corners=8pt,
    text width=3.85cm,
    minimum height=1.35cm,
    inner xsep=4pt,
    inner ysep=5pt,
    align=center
  },
  arrow/.style={->, line width=0.6pt, draw=black!75, shorten <=2pt, shorten >=2pt},
  looparrow/.style={arrow, rounded corners=4pt},
  branchlabel/.style={font=\footnotesize, fill=white, inner sep=2pt, align=center}
]
\node[draw=black!70, ellipse, minimum width=2.8cm, minimum height=0.65cm, align=center]
  (statement) at (0,1.65) {Fixed statement $s$};
\node[box] (initial) at (0,0) {\textbf{Initial response}\\LLM prover $P$\\$y_s=Q(\langle s\rangle,\theta)$\\$h_0=s\parallel y_s$};
\node[box] (challenge) at (5.15,0) {\textbf{Challenge}\\Human verifier $V$\\$x_m\sim\pi_V(\cdot\mid h_{m-1},i)$};
\node[box] (response) at (10.3,0) {\textbf{Response}\\LLM prover $P$\\$y_m=Q(h_{m-1}\parallel x_m,\theta)$};
\node[box] (update) at (10.3,-2.55) {\textbf{Append round}\\Human verifier $V$\\$h_m\leftarrow h_{m-1}\parallel x_m\parallel y_m$};
\node[box] (evaluate) at (5.15,-2.55) {\textbf{Check and decide}\\Human verifier $V$\\Check; update record\\$\omega\leftarrow U(h_m,i)$};
\node[decision] (route) at (0,-2.55) {\textbf{Route}\\$\omega$};
\node[terminal] (reject) at (0,-5) {\textbf{Reject}\\$\omega=\text{Reject}$ or\\$\omega=\text{Challenge},\ m=T$};
\node[terminal] (accept) at (5.15,-5) {\textbf{Accept}\\$\omega=\text{Accept}$};

\draw[arrow] (statement.south) -- (initial.north);
\draw[arrow] (initial.east) -- node[branchlabel, above, inner sep=1pt] {$m=1$} (challenge.west);
\draw[arrow] (challenge.east) -- (response.west);
\draw[arrow] (response.south) -- (update.north);
\draw[arrow] (update.west) -- (evaluate.east);
\draw[arrow] (evaluate.west) -- (route.east);
\draw[arrow] (route.south) -- (reject.north);
\draw[arrow] (route.south east) -- (accept.north);
\draw[looparrow] (route.west) -- (-2.45,-2.55) -- (-2.45,2.7)
  -- node[branchlabel, above] {$\omega=\text{Challenge},\ m<T$\\$m\leftarrow m+1$}
  (5.15,2.7) -- (challenge.north);
\end{tikzpicture}
\caption{Interactive Human--LLM Deliberation Protocol (Definition~\ref{def:deliberation_protocol}).}
\label{fig:protocol_diagram}
\end{figure}

\subsection{Running Example: Reviewing a Conference Schedule}
\label{sec:running_example}

To make Definition~\ref{def:deliberation_protocol} concrete, we follow one review task throughout the paper. An LLM proposes a conference schedule and claims that it has no conflicts, meaning that every session fits the recorded availability of its speakers and room and no two overlapping sessions share a speaker or room; back-to-back sessions are allowed. The schedule and records stay fixed during the review, so the claim can be decided by a finite set of time comparisons.

Suppose Lee is listed as both a keynote speaker and a panelist. The organizer asks, ``Can Lee attend both sessions?'' The LLM must point to the two entries. If the keynote runs from 10:00 to 10:30 and the panel from 10:15 to 10:45, comparing them exposes a conflict. The organizer checks this one answer without reviewing every session or building a replacement schedule.

One failed check is enough to reject the claim. Accepting it is harder, because the organizer can afford at most $T$ checks and passing them does not show that every requirement holds. The next subsection asks how much evidence such passed checks provide.

\subsection{Evidence Checks: An Idealized Example}
\label{sec:diagnostic_example}

A check whose result counts toward acceptance is an \emph{evidence check}, and a round that runs one is an \emph{evidence round}. For a passed check to count as evidence, a false claim must have a known chance of failing it, however the LLM answers. Random checks of the schedule give such a chance.

The claim of no conflicts breaks into a finite list of requirements: for each session, its time fits the recorded availability of its speakers and room; for each pair of sessions sharing a speaker or room, the two do not overlap. Let $J>1$ be the number of requirements and $s_{\mathrm{sched}}$ the claim that all of them hold. The organizer writes this list from the schedule before the dialogue begins, so the LLM's answers cannot change what is checked. The information state $i$ contains the schedule, the source records, the list, and the checking procedure.

In each evidence round, the organizer draws one requirement uniformly at random, independently of earlier rounds, and asks the LLM for the relevant records and time comparison. The check passes only if the cited records match the originals and the drawn requirement holds. If the draw is the requirement that Lee's two sessions do not overlap, the check fails whatever the LLM says, because the recorded times overlap.

In this subsection we assume that the organizer executes every check correctly; Section~\ref{sec:zkp} removes this assumption. If the schedule has a conflict, at least one requirement fails, and each draw selects it with probability at least $1/J$. To state this as a bound, write $\Pr^{s,P}$ for probability over one run of the protocol on statement $s$ with prover $P$, taken over the LLM's responses, the organizer's draws, and errors in executing the checks. Let $\mathcal F_{m-1}$ record the dialogue and draws before round $m$, and let $\mathsf{Pass}_m$ be the event that the round-$m$ check passes. If $s_{\mathrm{sched}}$ is false, every prover strategy $P^*$, including one that adapts to earlier rounds, satisfies
\[
\Pr^{s_{\mathrm{sched}},P^*}(\mathsf{Pass}_m\mid\mathcal F_{m-1})
\le 1-\frac{1}{J}.
\]
Such an inequality, which caps the chance that a check passes a false claim, is a \emph{false-pass bound}.

When the list is long, each pass is weak evidence. Acceptance then needs many passes, more as $J$ grows or the required confidence rises, and all of them must fit within the round budget $T$. If the schedule has no conflicts and the LLM always supplies the correct records, every check passes, and the claim is accepted exactly when enough rounds fit within $T$.

The bound relies on drawing from the complete list and checking against fixed records. Follow-up questions can direct attention to suspicious entries, but the same bound does not automatically cover them; the \hyperref[sec:anchored_review_example]{anchored construction} in Section~\ref{sec:zkp} mixes them with random draws and keeps a justified bound.
\section{Sequential Acceptance under Response and Execution Error}
\label{sec:zkp}

Section~\ref{sec:diagnostic_example} rested on two simplifications: an honest LLM always answers adequately, and the human executes every check correctly. In practice either can fail. In the scheduling review, the LLM may cite the wrong record for a session that has no conflict, and the organizer may misread a correct time comparison. This section keeps the protocol of Section~\ref{sec:protocol} and turns passed checks into a running evidence value that accounts for both kinds of error.

Each passed evidence check multiplies this value by $1/\bar q_m$, where $\bar q_m$ bounds the chance that the check, as the human actually executes it, passes a false claim; a failed check ends the review with Reject. The verifier accepts once the value reaches $1/\delta$. If the claim is false, the verifier accepts with probability at most $\delta$, however the LLM adapts its answers (Theorem~\ref{thm:anytime_soundness}). If the claim is true, Theorem~\ref{thm:conditional_completeness} bounds the probability of acceptance within the round budget from below, provided enough evidence can accumulate.

Section~\ref{sec:round_events} defines the per-round events and the information on which their probability bounds are conditioned. Section~\ref{sec:effective_verifier} states the conditions on the verifier and the threshold rule. Section~\ref{sec:anytime_soundness_sub} proves soundness and gives one way to construct the false-pass bound it requires. Section~\ref{sec:completeness_sub} proves completeness, and Section~\ref{sec:ip_connection} relates the result to classical interactive proofs.

\subsection{Per-Round Events and Information}
\label{sec:round_events}

Fix an evidence round $m$. Three events describe what can go right or wrong in it:
\begin{itemize}
\item $\mathcal A_m$: the LLM's response is \emph{adequate}, that is, admissible, truthful, checkable, and sufficient to pass the check when it is executed correctly, for every value of any randomness built into the checker;
\item $\mathcal E_m$: the human executes the check correctly;
\item $\mathsf{Pass}_m$: the check, as executed, returns Pass; this is the outcome that $U$ uses.
\end{itemize}
The two failures in the scheduling review above are failures of $\mathcal A_m$ and $\mathcal E_m$, respectively. The numbers $a_m$ and $c_m$ are lower bounds on the conditional probabilities of $\mathcal A_m$ and $\mathcal E_m$. Probabilities $\Pr^{s,P}$ are as in Section~\ref{sec:diagnostic_example}, and $\mathbb E^{s,P}$ is the corresponding expectation. Bounds on rejecting a true claim use an honest prover $P^{\mathrm H}$; bounds on accepting a false claim allow an arbitrary adaptive prover $P^*$.

\paragraph{What the verifier knows.}
Three \emph{observable histories} mark the stages of round $m$. $\mathcal F_{m-1}$ is everything the verifier has seen before choosing the round-$m$ challenge: the transcript, its private record, and the random bits already used. Using only $\mathcal F_{m-1}$, the verifier picks the challenge, decides whether the round is an evidence round, and fixes the ideal false-pass bound $q_m$. Once the response $y_m$ arrives, $\mathcal G_m$ adds the challenge $x_m$, the response, and any other information available before the check. After the check, $\mathcal F_m$ adds the outcome and the updated record. Formally, each is a $\sigma$-field on the probability space of one run: the collection of events whose occurrence is settled by the corresponding information.

We assume $\mathcal A_m\in\mathcal G_m$: whether the response is adequate is settled once $x_m$ and $y_m$ are available, although the verifier may have no efficient way to decide it. In the scheduling review, $\mathcal G_m$ describes the point at which the organizer has drawn the requirement that Lee's two sessions do not overlap and the LLM has cited the two entries, but the organizer has not yet compared them with the source records. The execution bound in Section~\ref{sec:effective_verifier} conditions on this point, so it covers the specific response the LLM has just given.

Write $r_m^\circ$ for the probability that the prescribed check returns Pass when executed correctly, given the history $h_{m-1}$ and the pair $(x_m,y_m)$; it includes any randomness built into the checker. On false statements, $q_m$ caps the average of $r_m^\circ$ over the challenge policy and the adaptive response. Uniform sampling in Section~\ref{sec:diagnostic_example} gives $q_m=1-1/J$.

\paragraph{What the analysis also tracks.}
The verifier may not observe whether an earlier response was adequate or whether it executed an earlier check correctly, but the proofs condition on these events. The analysis therefore adds them to the observable history. After round $m$, the resulting \emph{analysis history}, which contains everything the verifier observed together with these hidden outcomes, is
\begin{equation}
\mathcal J_m
:=
\sigma\!\left(\mathcal F_m,
\mathcal A_1,\mathcal E_1,\ldots,
\mathcal A_m,\mathcal E_m\right),
\label{eq:augmented_analysis_filtration}
\end{equation}
and its pre-check counterpart for round $m$ is
\begin{equation}
\mathcal K_m:=\mathcal J_{m-1}\vee\mathcal G_m.
\label{eq:full_precheck_field}
\end{equation}
Here $\sigma(\cdot)$ denotes the $\sigma$-field generated by its arguments, and $\vee$ joins two $\sigma$-fields. Thus $\mathcal J_{m-1}$ and $\mathcal K_m$ are the analysis counterparts of $\mathcal F_{m-1}$ and $\mathcal G_m$. The challenge and $q_m$ use only observable histories, while the reliability bounds must hold conditional on the analysis histories. Conditioning on the analysis histories lets each round's events depend on anything that happened earlier, including unobserved errors.

\begin{condition}[Response reliability of the honest prover]
\label{condition:response_reliable_prover}
Fix a true statement $s\in L$, information state $i\in\mathcal I$, round budget $T$, and numerical lower bounds $a_m\in[0,1]$ chosen before execution. A possibly stochastic honest prover $P^{\mathrm H}$ is \emph{response-reliable} if its initial response is admissible almost surely and, at each reached evidence round,
\[
\Pr^{s,P^{\mathrm H}}(\mathcal A_m\mid\mathcal J_{m-1})\ge a_m
\]
at every reachable history, averaging over the current challenge and the prover's randomness. Clarification responses are admissible almost surely.
\end{condition}

\subsection{Effective Verifier and Threshold Rule}
\label{sec:effective_verifier}

We combine the ideal false-pass bound $q_m$ with the execution bound $c_m$ to obtain the implemented false-pass bound and specify how passed checks accumulate evidence. The conditions below apply to a fixed task and verifier information state, with prescribed challenge and checking policies executable within the verifier's resources.

\begin{assumption}[Effective Verifier]
\label{assumption:effective_verifier_confidence}
Fix target soundness $\delta\in(0,1)$ and round budget $T$. The execution lower bounds $c_m\in[0,1]$, $m\le T$, are deterministic and fixed before execution. The verifier also uses an $\mathcal F_{m-1}$-measurable ideal false-pass bound $q_m\in[0,1)$. Define
\begin{equation}
\bar q_m:=1-c_m(1-q_m).
\label{eq:implemented_false_pass}
\end{equation}
On every evidence round, $0<\bar q_m<1$. The following conditions hold.
\begin{enumerate}
\item\label{item:execution} \textbf{Human execution reliability.} For every false statement with every adaptive admissible prover, and for every true statement with every response-reliable honest prover, every reached evidence round $m\le T$ satisfies, at every reachable pre-check history,
\[
\Pr(\mathcal E_m\mid\mathcal K_m)\ge c_m,
\qquad
\Pr(\mathcal E_m\cap\mathsf{Pass}_m\mid\mathcal K_m)
=\Pr(\mathcal E_m\mid\mathcal K_m)r_m^\circ.
\]
The inequality says that, whatever the response, the human executes the check correctly with probability at least $c_m$. The identity formalizes what correct execution means: when the human executes the check correctly, the recorded outcome is the outcome of the prescribed check. In the scheduling review, an organizer who compares the times correctly records Pass exactly when the cited records match and the sessions do not overlap.

\item\label{item:validity} \textbf{False-statement validity under correct execution.} For every false statement $s\notin L$, every adaptive admissible prover $P^*$, and every reached evidence round $m$, at every reachable history,
\[
\mathbb E^{s,P^*}[r_m^\circ\mid\mathcal F_{m-1}]
\le q_m.
\]
The expectation averages over the current challenge, the prover's response, and the checker's randomness. The bound need hold only for the challenge policy as a whole, which may be nonuniform and history-adaptive, not for each individual question.

\item\label{item:correctness} \textbf{True-statement correctness.} For every $s\in L$ and every response-reliable honest prover $P^{\mathrm H}$ (Condition~\ref{condition:response_reliable_prover}),
\[
\mathcal A_m\cap\mathcal E_m\Longrightarrow\mathsf{Pass}_m
\]
on every reached evidence round: an adequate response to a correctly executed check never fails a true claim.

\item\label{item:threshold} \textbf{Evidence process and verdict.} Set $M_0:=1$ before receiving any prover message. On a clarification round, set $q_m=\bar q_m=r_m^\circ=1$ and let $\mathcal A_m$, $\mathcal E_m$, and $\mathsf{Pass}_m$ occur surely by convention. Keep the deterministic lower bounds $a_m,c_m$ unchanged and leave $M_m=M_{m-1}$. On a reached evidence round, update
\begin{equation}
M_m
=M_{m-1}
\frac{\mathbf1_{\{\mathsf{Pass}_m\}}}{\bar q_m}.
\label{eq:certificate_update}
\end{equation}
A failed evidence check therefore sets $M_m=0$ and causes immediate Reject. After every reached round,
\[
U(h_m,i)=
\begin{cases}
\mathrm{Reject}, & \text{the evidence check fails},\\
\mathrm{Accept}, & \text{the check does not fail and }M_m\ge1/\delta,\\
\mathrm{Challenge}, & \text{otherwise}.
\end{cases}
\]
If the last case occurs at $m=T$, the protocol outputs Reject. If the run stops after $m'<T$ rounds, set $M_m=M_{m'}$ and $\mathcal F_m=\mathcal F_{m'}$ for $m'<m\le T$, so that both are defined for every $m\le T$.

\item\label{item:progress} \textbf{Progress within the round budget.} There exists a deterministic $K\le T$ such that, on every honest execution in which the first $K$ reached responses are adequate and their checks are correctly executed, either the protocol has accepted earlier or
\begin{equation}
\sum_{m=1}^{K}-\log\bar q_m\ge\log(1/\delta).
\label{eq:diagnostic_progress}
\end{equation}
Without this item, even an error-free honest run could use up the round budget before reaching $1/\delta$, for example when each check is weak and $T$ is small.
\end{enumerate}
\end{assumption}

A passed evidence check adds $-\log\bar q_m$ to $\log M_m$. For fixed $q_m$, a larger execution bound $c_m$ lowers $\bar q_m$ and so earns more evidence per pass. $M_m$ is not a posterior probability that the claim is true: for a false statement under the conditions above, it is a test supermartingale, and hence an e-process \citep{shafer2011test,ramdas2023gametheoretic}; Theorem~\ref{thm:anytime_soundness} rests on this property. Clarification leaves $M_m$ unchanged but can guide later challenges; in the scheduling review, asking the LLM to explain an entry can help the organizer decide what to inspect next.

The two guarantees use these conditions differently. Soundness uses items~\ref{item:execution}, \ref{item:validity}, and~\ref{item:threshold} and does not require adequate responses. Completeness uses items~\ref{item:execution}, \ref{item:correctness}, \ref{item:threshold}, and~\ref{item:progress} together with Condition~\ref{condition:response_reliable_prover}; it does not need item~\ref{item:validity}, which concerns only false statements. Thus $c_m$ enters both guarantees, while $a_m$ enters only completeness.

\subsection{Anytime-Valid Soundness}
\label{sec:anytime_soundness_sub}

The following result applies the standard test-supermartingale argument to the implemented false-pass bounds.

\begin{theorem}[Anytime-valid soundness]
\label{thm:anytime_soundness}
Suppose items~\ref{item:execution}, \ref{item:validity}, and~\ref{item:threshold} of Assumption~\ref{assumption:effective_verifier_confidence} hold. For every $s\notin L$ and every adaptive admissible prover strategy $P^*$,
\begin{equation}
\Pr^{s,P^*}\!\left(\exists m\le T:M_m\ge1/\delta\right)
\le\delta.
\label{eq:anytime_soundness}
\end{equation}
Consequently,
\[
\sup_{P^*}\Pr^{s,P^*}(\Omega=\mathrm{Accept})\le\delta.
\]
Because~\eqref{eq:anytime_soundness} controls the whole run, the bound remains valid if the review stops at any round chosen from the observed history, provided acceptance still requires $M_m\ge1/\delta$.
\end{theorem}

\begin{proof}
Fix $s\notin L$ and $P^*$, and write $\Pr$ and $\mathbb E$ for $\Pr^{s,P^*}$ and $\mathbb E^{s,P^*}$. Let $\mathsf{Fail}_m$ denote the complement of $\mathsf{Pass}_m$. Whether round $m$ is a reached evidence round is determined by $\mathcal F_{m-1}$. If it is not, because it is a clarification round or the run has already stopped, item~\ref{item:threshold} gives $M_m=M_{m-1}$.

On a reached evidence round, dropping the event that an incorrectly executed check fails gives
\begin{align*}
\Pr(\mathsf{Fail}_m\mid\mathcal F_{m-1})
&\ge \Pr(\mathcal E_m\cap\mathsf{Fail}_m\mid\mathcal F_{m-1})\\
&=\mathbb E\!\left[
\Pr(\mathcal E_m\mid\mathcal K_m)(1-r_m^\circ)
\middle|\mathcal F_{m-1}\right]\\
&\ge c_m\mathbb E[1-r_m^\circ\mid\mathcal F_{m-1}]\\
&\ge c_m(1-q_m).
\end{align*}
The equality uses the identity in item~\ref{item:execution} and $\mathcal F_{m-1}\subseteq\mathcal K_m$; the two inequalities that follow use the execution bound in item~\ref{item:execution} and item~\ref{item:validity}. Hence $\Pr(\mathsf{Pass}_m\mid\mathcal F_{m-1})\le\bar q_m$. Since $M_{m-1}$ and $\bar q_m$ are $\mathcal F_{m-1}$-measurable,~\eqref{eq:certificate_update} gives
\[
\mathbb E[M_m\mid\mathcal F_{m-1}]
=\frac{M_{m-1}}{\bar q_m}\Pr(\mathsf{Pass}_m\mid\mathcal F_{m-1})
\le M_{m-1}.
\]
Thus $(M_m)_{m=0}^T$ is a nonnegative supermartingale with $M_0=1$, and Ville's inequality \citep{shafer2011test,ramdas2023gametheoretic} yields
\[
\Pr\!\left(\max_{0\le m\le T}M_m\ge1/\delta\right)
\le\delta.
\]
Acceptance requires $M_m\ge1/\delta$ for some $m\le T$, which proves~\eqref{eq:anytime_soundness}. The bound holds for each $P^*$, so it also holds for the supremum.
\end{proof}

The per-round bound $\bar q_m$ is tight: equality holds when correct execution has probability exactly $c_m$ at every pre-check history, a correctly executed check passes with average probability exactly $q_m$, and an incorrectly executed check always passes.

\paragraph{An anchored challenge construction.}
\phantomsection
\label{sec:anchored_review_example}
The soundness theorem needs a justified false-pass bound but does not fix how challenges are chosen. One way to keep such a bound while following the dialogue is to use, in round $m$, a reference check with a probability $\lambda_m^{\mathrm{anc}}$ chosen before the round. Under correct execution, the reference check must catch a false claim with probability at least $\rho_m$ against every admissible adaptive prover. Even if the other, dialogue-driven challenges catch nothing, the mixture gives
\[
q_m=1-\lambda_m^{\mathrm{anc}}\rho_m,
\qquad
\bar q_m=1-c_m\lambda_m^{\mathrm{anc}}\rho_m.
\]
In the scheduling example, a uniformly random requirement from the complete list gives $\rho_m=1/J$, so the organizer can pursue apparent conflicts while keeping this random component. Appendix~\ref{app:anchored_validity} gives the formal conditions and proof.

\subsection{Completeness within the Round Budget}
\label{sec:completeness_sub}

\begin{theorem}[Completeness within the round budget]
\label{thm:conditional_completeness}
For a true statement $s\in L$, suppose the honest prover $P^{\mathrm H}$ is response-reliable (Condition~\ref{condition:response_reliable_prover}) and items~\ref{item:execution}, \ref{item:correctness}, \ref{item:threshold}, and~\ref{item:progress} of Assumption~\ref{assumption:effective_verifier_confidence} hold. Let $K\le T$ be the round bound in item~\ref{item:progress}. Then
\begin{equation}
\Pr^{s,P^{\mathrm H}}(\Omega=\mathrm{Accept})
\ge\prod_{m=1}^{K}a_mc_m.
\label{eq:completeness_product}
\end{equation}
In particular, for a false-rejection target $\gamma\in(0,1)$, if $\prod_{m=1}^{K}a_mc_m\ge1-\gamma$, then the protocol has completeness at least $1-\gamma$.
\end{theorem}

\begin{proof}
Write $\Pr$ for $\Pr^{s,P^{\mathrm H}}$. Let $\mathsf H_0$ denote the whole sample space and $\mathsf H_k:=\bigcap_{m=1}^{k}(\mathcal A_m\cap\mathcal E_m)$. After early acceptance, declare all unused adequacy and execution events true; after rejection, declare them false. Extend the $\sigma$-fields by the corresponding terminal records, so these events remain measurable.

Fix $k\le K$. On $\mathsf H_{k-1}$, the protocol is either still active at round $k$ or has already accepted. It cannot have rejected: a failed check is excluded by item~\ref{item:correctness}; a malformed or over-length message is excluded because $\mathcal A_m$ includes admissibility on evidence rounds and Condition~\ref{condition:response_reliable_prover} makes the initial and clarification responses admissible almost surely; and the round budget cannot run out before round $k\le T$.

Whether the run is active at round $k$ is determined by $\mathcal F_{k-1}\subseteq\mathcal J_{k-1}$. On the active part of $\mathsf H_{k-1}$, if round $k$ is a clarification round, $\mathcal A_k\cap\mathcal E_k$ occurs by convention. If it is an evidence round, then $\mathcal A_k\in\mathcal G_k\subseteq\mathcal K_k$ and $\mathcal J_{k-1}\subseteq\mathcal K_k$, so item~\ref{item:execution} and Condition~\ref{condition:response_reliable_prover} give
\[
\Pr(\mathcal A_k\cap\mathcal E_k\mid\mathcal J_{k-1})
=\mathbb E\!\left[
\mathbf1_{\mathcal A_k}\Pr(\mathcal E_k\mid\mathcal K_k)
\mid\mathcal J_{k-1}\right]
\ge c_k\Pr(\mathcal A_k\mid\mathcal J_{k-1})
\ge a_kc_k.
\]
On the already-accepted part, the conditional probability is one by convention. Hence $\Pr(\mathcal A_k\cap\mathcal E_k\mid\mathcal J_{k-1})\ge a_kc_k$ on $\mathsf H_{k-1}$. Because $\mathsf H_{k-1}\in\mathcal J_{k-1}$,
\[
\Pr(\mathsf H_k)
=\mathbb E\!\left[
\mathbf1_{\mathsf H_{k-1}}
\Pr(\mathcal A_k\cap\mathcal E_k\mid\mathcal J_{k-1})
\right]
\ge a_kc_k\Pr(\mathsf H_{k-1}),
\]
and induction gives $\Pr(\mathsf H_K)\ge\prod_{m=1}^{K}a_mc_m$.

On $\mathsf H_K$, suppose the protocol has not accepted before round $K$. Then every evidence check in rounds $1,\dots,K$ has passed, and clarification rounds have left $M_m$ unchanged. By~\eqref{eq:certificate_update} and item~\ref{item:progress},
\[
\log M_K=\sum_{m=1}^{K}-\log\bar q_m
\ge\log(1/\delta),
\]
where each clarification round contributes $-\log1=0$. The protocol therefore accepts by round $K$, so $\mathsf H_K\subseteq\{\Omega=\mathrm{Accept}\}$, which proves~\eqref{eq:completeness_product}.
\end{proof}

Running out the round budget counts as a failure in this bound: outside $\mathsf H_K$, the proof allows any inadequate response or execution error to cause rejection, leave the evidence short of the threshold, or use up the round budget.

\subsection{Connection to Interactive Proofs}
\label{sec:ip_connection}

When $a_m=c_m=1$, we have $\bar q_m=q_m$ and $\prod_{m=1}^{K}a_mc_m=1$, so Theorems~\ref{thm:anytime_soundness} and~\ref{thm:conditional_completeness} give soundness at level $\delta$ and acceptance by round $K$ with probability one. The protocol then has the form of a classical interactive proof. Corollary~\ref{thm:conditional_ip} states when the protocol with general $a_m$ and $c_m$ is one. The corollary treats the verifier as an algorithm; an actual human would require separate formalization as a uniform probabilistic polynomial-time verifier.

\begin{corollary}[Conditional IP instantiation]
\label{thm:conditional_ip}
Write $i=(i_{\mathrm{pub}},i_{\mathrm{priv}})$ and let $x=\langle s,i_{\mathrm{pub}}\rangle$ define a promise problem $(\mathcal Y,\mathcal N)$, with verifier-only state $i_{\mathrm{priv}}$ fixed or uniformly generated by $V$. Measure input size by $\kappa=|s|+\operatorname{size}(i_{\mathrm{pub}})$, the length of an encoding of $x$. Suppose $V(1^\kappa,x)$ is uniform probabilistic polynomial time, while $T(\kappa)$, communication, and parameter representations are polynomially bounded. If the hypotheses of Theorems~\ref{thm:anytime_soundness} and~\ref{thm:conditional_completeness} hold uniformly in $\kappa$ and, for some polynomial $p$,
\[
\prod_{m=1}^{K}a_mc_m-\delta\ge\frac{1}{p(\kappa)},
\]
then the instantiation is an interactive proof with completeness at least $\prod_{m=1}^{K}a_mc_m$ and soundness at most $\delta$ against every unbounded admissible prover; polynomial-time prover soundness instead gives an interactive argument. Standard amplification yields the conventional $2/3$--$1/3$ parameters.
\end{corollary}

Theorems~\ref{thm:anytime_soundness} and~\ref{thm:conditional_completeness} hold for any fixed bounds, but they do not say how long a review should be. Section~\ref{sec:fallibility} takes the bounds as given and asks which review lengths meet both targets within the round budget, then examines how effort, cognitive load, and expertise affect $c_m$ and thus which lengths qualify.
\section{When Limited Human Resources Suffice}
\label{sec:fallibility}

Each passed check adds evidence against a false claim. Each round, however, also needs an adequate response from the LLM and a correctly executed check from the human, so a true claim can fail in any round. This section fixes a false-acceptance target $\delta$, the same $\delta$ that sets the acceptance threshold $1/\delta$ in Section~\ref{sec:zkp}, and a false-rejection target $\gamma$. Section~\ref{sec:effective_region} identifies the review lengths at which the bounds guarantee both targets within the round budget $T$ of Section~\ref{sec:protocol}. Section~\ref{sec:effort} models how the execution bounds depend on the effort spent on each check, its cognitive load, and the reviewer's expertise. Section~\ref{sec:separation} then compares a review split into local checks with a single one-shot check under the same budget, and states when the bounds certify the local review but not the one-shot check.

\subsection{The Effective Region}
\label{sec:effective_region}

Until Section~\ref{sec:effort}, the rules for choosing challenges and checking responses are fixed, and only the review length varies; in the scheduling review, the timetable and requirement list stay the same. We compare two ways of stopping a review that follows these rules. The protocol of Section~\ref{sec:zkp}, denoted $\Pi^{\mathrm{ES}}$ (for \emph{early stopping}), rejects at the first failed check and accepts as soon as the evidence value reaches $1/\delta$. For a length $n\le T$ fixed in advance, the \emph{fixed-length} protocol $\Pi^{(n)}$ rejects at any failed check and accepts only after all $n$ checks pass.

To obtain closed-form criteria, we treat every round as an evidence round (a clarification round would contribute a factor of one) and use the same ideal false-pass bound $q\in[0,1)$ in every round. The response bounds $a_i$ and execution bounds $c_i$ may vary by round and are fixed before execution. We assume that Condition~\ref{condition:response_reliable_prover} and items~\ref{item:execution}--\ref{item:correctness} of Assumption~\ref{assumption:effective_verifier_confidence} hold for these rules with these bounds. The implemented false-pass bound~\eqref{eq:implemented_false_pass} then becomes
\begin{equation}
\bar q_i:=1-c_i(1-q).
\label{eq:false_survival_round}
\end{equation}
As in Assumption~\ref{assumption:effective_verifier_confidence}, we require $0<\bar q_i<1$; since $q<1$, this holds exactly when $c_i>0$ and $c_i(1-q)<1$. Multiplying over the first $n$ rounds gives
\begin{equation}
R_n:=\prod_{i=1}^{n}\bar q_i,
\qquad
C_n:=\prod_{i=1}^{n}a_ic_i.
\label{eq:cumulative_burdens}
\end{equation}
Here $R_n$ bounds the probability that a false claim passes all $n$ checks, and $C_n$ bounds from below the probability that a true claim does. For the fixed-length protocol $\Pi^{(n)}$, they therefore bound false and true acceptance (Proposition~\ref{prop:fallible_execution_bounds}). For $\Pi^{\mathrm{ES}}$, a run in which every check passes has evidence value $M_n=1/R_n$ after $n$ rounds, so it reaches the threshold $1/\delta$ at
\begin{equation}
K_\delta:=\min\{n\le T:R_n\le\delta\},
\label{eq:first_certification_round}
\end{equation}
when this set is nonempty. If it is empty, even a run in which every check passes cannot reach the threshold, and $\Pi^{\mathrm{ES}}$ rejects at $T$.

\begin{definition}[\texorpdfstring{$(\delta,\gamma)$}{(delta,gamma)}-certification]
\label{def:fallibility_validity}
Let $\Pi$ be a protocol whose stopping rule is fixed in advance, and write $\Pr_{\Pi}^{s,P}$ for probabilities over one run of $\Pi$ on statement $s$ with prover $P$; we write $s^-$ for a false statement and $s^+$ for a true one. The bounds $q$ and $(a_i,c_i)_{i\le T}$ \emph{certify} soundness at level $\delta$ if
\[
\Pr_{\Pi}^{s^-,P^*}(\Omega=\mathrm{Accept})\le\delta
\]
for every $s^-\notin L$ and every adaptive admissible prover $P^*$, and they certify completeness at level $1-\gamma$ if
\[
\Pr_{\Pi}^{s^+,P^{\mathrm H}}(\Omega=\mathrm{Reject})\le\gamma
\]
for every $s^+\in L$ and every honest prover $P^{\mathrm H}$ satisfying Condition~\ref{condition:response_reliable_prover} with the bounds $a_i$. Rejection includes failed checks and running out the round budget. Each inequality must hold under every distribution of runs that satisfies the conditions of this section with these bounds, not only under the one that actually occurs. If both are certified, $\Pi$ is \emph{$(\delta,\gamma)$-certified}.
\end{definition}

\begin{proposition}[Error bounds for the two stopping rules]
\label{prop:fallible_execution_bounds}
Let $s^-\notin L$ with an adaptive admissible prover $P^*$, and let $s^+\in L$ with an honest prover $P^{\mathrm H}$ satisfying Condition~\ref{condition:response_reliable_prover} with the bounds $a_i$.
\begin{enumerate}
\item For every $n\le T$, the fixed-length protocol satisfies
\[
\Pr_{\Pi^{(n)}}^{s^-,P^*}(\Omega=\mathrm{Accept})\le R_n,
\qquad
\Pr_{\Pi^{(n)}}^{s^+,P^{\mathrm H}}(\Omega=\mathrm{Reject})\le1-C_n.
\]
\item If $K_\delta$ is defined, the early-stopping protocol satisfies
\begin{align}
\Pr_{\Pi^{\mathrm{ES}}}^{s^-,P^*}(\Omega=\mathrm{Accept})
&\le R_{K_\delta}\le\delta,
\label{eq:false_accept_prod}\\
\Pr_{\Pi^{\mathrm{ES}}}^{s^+,P^{\mathrm H}}(\Omega=\mathrm{Reject})
&\le1-C_{K_\delta}.
\label{eq:false_reject_ub}
\end{align}
\end{enumerate}
Consequently, $\Pi^{(n)}$ is $(\delta,\gamma)$-certified if $R_n\le\delta$ and $C_n\ge1-\gamma$, and $\Pi^{\mathrm{ES}}$ is $(\delta,\gamma)$-certified whenever $K_\delta$ is defined and
\begin{equation}
C_{K_\delta}\ge1-\gamma.
\label{eq:completeness_condition}
\end{equation}
\end{proposition}

\begin{proof}
Let $\mathsf S_n:=\bigcap_{i=1}^{n}\mathsf{Pass}_i$ be the event that the first $n$ checks pass. The fixed-length protocol $\Pi^{(n)}$ accepts exactly on $\mathsf S_n$. The early-stopping protocol accepts at the first round with $M_n\ge1/\delta$. Since $M_n=1/R_n$ as long as every check has passed, it cannot accept before $K_\delta$, and it accepts whenever $\mathsf S_{K_\delta}$ occurs. In both cases it suffices to bound $\Pr(\mathsf S_n)$, with $n=K_\delta$ for $\Pi^{\mathrm{ES}}$.

For the false statement, on $\mathsf S_{i-1}$ round $i$ is reached, and the proof of Theorem~\ref{thm:anytime_soundness} uses only items~\ref{item:execution} and~\ref{item:validity} to show $\Pr(\mathsf{Pass}_i\mid\mathcal F_{i-1})\le\bar q_i$ at every reached round. Since $\mathsf S_{i-1}\in\mathcal F_{i-1}$,
\[
\Pr(\mathsf S_i)
=\mathbb E\!\left[\mathbf 1_{\mathsf S_{i-1}}\Pr(\mathsf{Pass}_i\mid\mathcal F_{i-1})\right]
\le\bar q_i\Pr(\mathsf S_{i-1}),
\]
and induction gives $\Pr(\mathsf S_n)\le R_n$. For $\Pi^{\mathrm{ES}}$, $R_{K_\delta}\le\delta$ by the definition of $K_\delta$.

For the true statement, let $\mathsf H_n:=\bigcap_{i=1}^{n}(\mathcal A_i\cap\mathcal E_i)$. Item~\ref{item:correctness} gives $\mathsf H_n\subseteq\mathsf S_n$, and the induction in the proof of Theorem~\ref{thm:conditional_completeness}, which uses Condition~\ref{condition:response_reliable_prover} and items~\ref{item:execution} and~\ref{item:correctness}, gives $\Pr(\mathsf H_n)\ge C_n$. Since $\mathsf S_n$ leads to acceptance, the rejection probability is at most $1-\Pr(\mathsf S_n)\le1-C_n$.
\end{proof}

\begin{definition}[Effective region]
\label{def:certifiable_effective_region}
The \emph{effective region} is the set of review lengths at which the product bounds meet both targets,
\begin{equation}
\mathsf{ER}_{\delta,\gamma}
:=
\left\{n\in\{1,\ldots,T\}:
R_n\le\delta
\ \text{and}\
C_n\ge1-\gamma
\right\}.
\label{eq:exact_effective_region}
\end{equation}
When this set is nonempty, write
$n_{\min}:=\min\mathsf{ER}_{\delta,\gamma}$ and
$n_{\max}:=\max\mathsf{ER}_{\delta,\gamma}$.
\end{definition}

By Proposition~\ref{prop:fallible_execution_bounds}, the bounds certify $\Pi^{(n)}$ for every $n$ in the effective region. A length outside the region may still meet the actual error targets, but the product bounds do not show it.

\begin{proposition}[Effective-region characterization]
\label{prop:effective_region_characterization}
The effective region is nonempty if and only if $K_\delta$ is defined and condition~\eqref{eq:completeness_condition} holds. In that case,
\begin{equation}
\mathsf{ER}_{\delta,\gamma}
=\{n_{\min},n_{\min}+1,\ldots,n_{\max}\},
\qquad
n_{\min}=K_\delta,
\label{eq:effective_region_interval}
\end{equation}
and the bounds certify the early-stopping protocol $\Pi^{\mathrm{ES}}$ as well as every fixed-length protocol $\Pi^{(n)}$ with $n\in\mathsf{ER}_{\delta,\gamma}$.
\end{proposition}

\begin{proof}
Since $\bar q_i\in(0,1)$ and $a_ic_i\in[0,1]$, both $R_n$ and $C_n$ are nonincreasing in $n$. Hence $\{n\le T:R_n\le\delta\}$ is $\{K_\delta,\ldots,T\}$ when $K_\delta$ is defined and empty otherwise, while $\{n\le T:C_n\ge1-\gamma\}$ is an initial segment $\{1,\ldots,N\}$, possibly empty. Their intersection is nonempty exactly when $K_\delta\le N$, that is, when $C_{K_\delta}\ge1-\gamma$, and it then equals $\{K_\delta,\ldots,N\}$. The certification claims follow from Proposition~\ref{prop:fallible_execution_bounds}.
\end{proof}

Within the region, moving from $n_{\min}$ toward $n_{\max}$ lowers the false-acceptance bound $R_n$ at the cost of the true-acceptance bound $C_n$; $n_{\max}$ is the longest length fixed in advance that the bounds still certify. Figure~\ref{fig:effective_region} shows the two boundaries for one fixed bound sequence: the soundness target sets the shortest certified length, and the completeness target sets the longest.

\begin{figure}[!htbp]
\centering
\includegraphics[width=\textwidth]{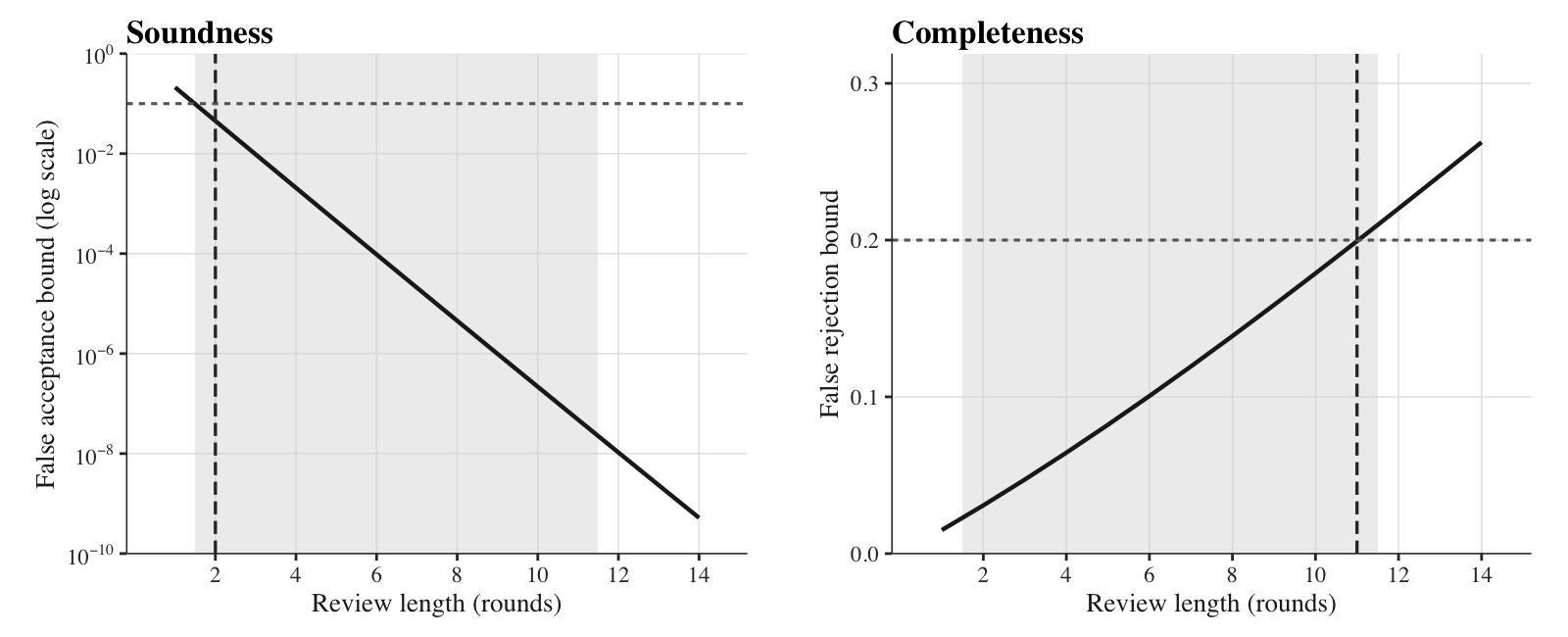}
\caption{Effective region for $T=14$, $q=0.20$, $\delta=0.10$, $\gamma=0.20$, $a_i=1$, and $c_i=0.985-0.001(i-1)$. Left: $R_n$ (log scale) against $\delta$. Right: $1-C_n$ against $\gamma$. The shaded overlap is $\mathsf{ER}_{\delta,\gamma}=\{2,\ldots,11\}$.}
\label{fig:effective_region}
\end{figure}

Appendix~\ref{app:worked_certificates} applies these product bounds to a matrix-product check with a nonempty effective region\ifextended{} and to a scheduling review whose required number of checks exceeds its budget\fi.

\subsection{Effort, Cognitive Load, and Expertise}
\label{sec:effort}

Section~\ref{sec:effective_region} treated the execution bounds $c_i$ as given. This subsection models how they depend on the resources a person brings to the review.

Let $E$ be the total human-review budget, and let $e_i\ge0$ be the effort spent on round $i$, counting all work for its check, with $\sum_{i=1}^{n}e_i\le E$. Let $\xi\ge0$ summarize task-relevant expertise and tool access. The cognitive load $\ell_i\ge0$ of round $i$ measures how much material a person must read and keep in mind to carry out its check, including the context it requires; it may depend on how the evidence is presented. Loads are fixed before execution. We model the execution bound as
\begin{equation}
c_i=f(e_i,\xi)
\exp\!\left(-\frac{\ell_i}{r(\xi)}\right).
\label{eq:ci_hardness}
\end{equation}
Here $f\in[0,1]$ is weakly increasing in effort and expertise, and the positive scale $r$ is weakly increasing in expertise. Equation~\eqref{eq:ci_hardness} is an illustrative model and does not by itself establish execution bounds. For each review design, the resulting $c_i$ must still be justified as bounds satisfying Assumption~\ref{assumption:effective_verifier_confidence}, including when the LLM responds adversarially, and $q$ needs its own justification.

In Section~\ref{sec:effective_region}, a shorter review consisted of the first rounds of a longer one, with the same bounds in each round. Under a fixed budget this no longer holds: each length $n$ divides $E$ differently, so it has its own bounds $a_i^{(n)},c_i^{(n)}$, justified in advance. Under equal allocation $e_i^{(n)}=E/n$, a longer review gives each check less effort, so $c_i^{(n)}$ falls as $n$ grows. Length therefore has two costs for a true claim: more chances to fail, and less effort for each of them. For a budget $E$, define
\begin{equation}
\mathsf{ER}_{\delta,\gamma}(E)
:=
\left\{n\le T:
\begin{array}{l}
\prod_{i=1}^{n}\bigl(1-c_i^{(n)}(1-q)\bigr)\le\delta,\\
\prod_{i=1}^{n}a_i^{(n)}c_i^{(n)}\ge1-\gamma
\end{array}
\right\}.
\label{eq:effort_effective_region}
\end{equation}
For each $n$ in this set, Proposition~\ref{prop:fallible_execution_bounds} shows that the bounds certify $\Pi^{(n)}$ under its allocation. Because the bounds now change with $n$, the two products need not be monotone in $n$, so Proposition~\ref{prop:effective_region_characterization} does not apply and the certified lengths need not form an interval. Under equal allocation, with all other parameters fixed, increasing $E$ keeps every certified length certified, because $f$ is nondecreasing in effort.

Figure~\ref{fig:effort_budget} shows this region for one parameterization. The shortest certified length stays at two for every budget, while the longest grows from four at $E=20$ to ten at $E=180$. Extra effort extends the region only at its upper end: the false-acceptance bound already meets $\delta$ at two rounds, and higher $c_i^{(n)}$ slows the decline of the true-acceptance bound. The loads are held fixed here; Section~\ref{sec:separation} uses the load term to compare local and one-shot checks.

\begin{figure}[t]
\centering
\includegraphics[width=\textwidth]{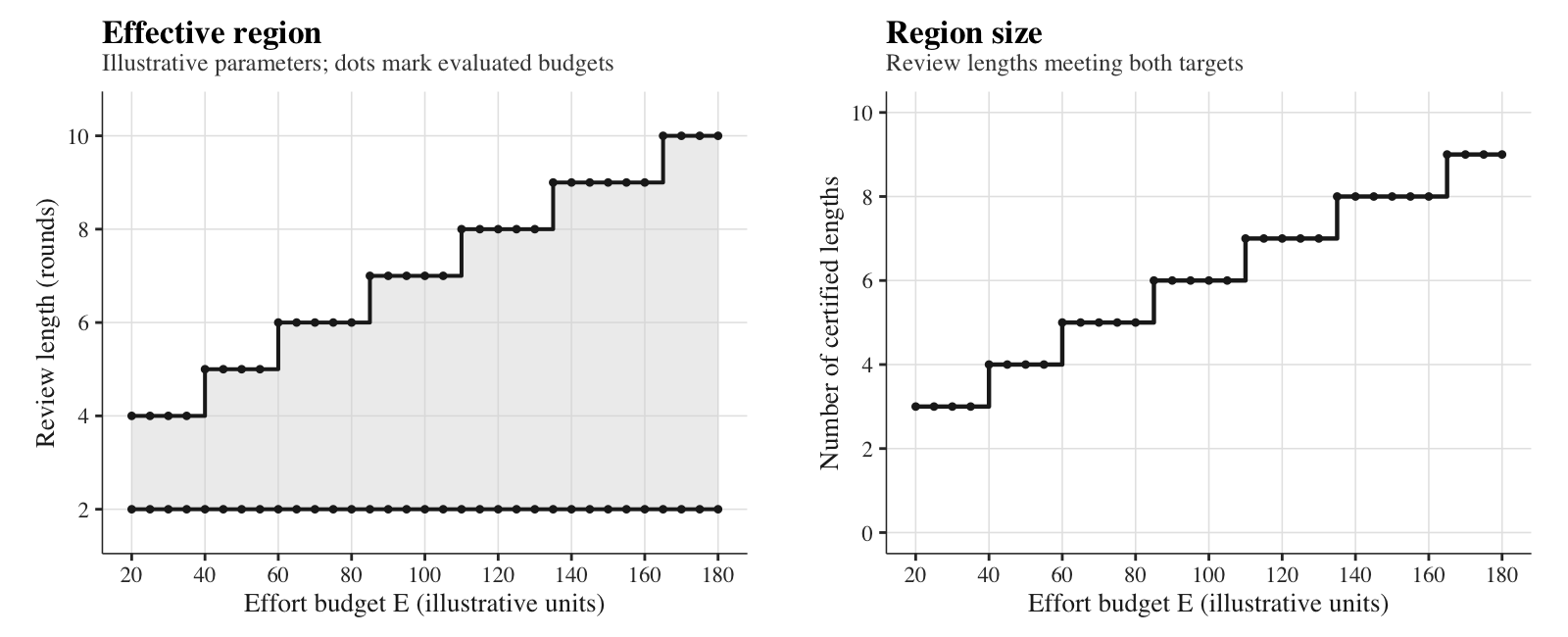}
\caption{Effective region $\mathsf{ER}_{\delta,\gamma}(E)$ as the budget $E$ grows, with $E$ split equally across rounds. Left: shortest and longest certified lengths. Right: number of certified lengths. $T$, $q$, $\delta$, $\gamma$, and $a_i$ as in Figure~\ref{fig:effective_region}; $c_i$ follows the effort model with $\xi=10$ (Appendix~\ref{app:fig3_reproducibility}).}
\label{fig:effort_budget}
\end{figure}

\subsection{Conditional Advantage of Decomposition}
\label{sec:separation}

For a long proof, a complete check may require reading and keeping in mind more steps than the human's budget allows, or a strong execution bound for it may be difficult to justify. We therefore compare two ways of reviewing the same fixed claim, candidate answer, and source records: the local review of Section~\ref{sec:effort} and a single \emph{one-shot} check of the whole claim, marked by the subscript $\mathrm{os}$. Both use the same verifier tools, the same computation and communication caps, and the same total budget $E$; the one-shot check receives effort $e_{\mathrm{os}}\le E$, counted in the same way as $e_i$. The one-shot check has its own ideal false-pass bound $q_{\mathrm{os}}$, justified separately from the local bound $q$; a complete exact check can have $q_{\mathrm{os}}=0$.

Exactness alone does not make a one-shot check reliable: with $q_{\mathrm{os}}=0$, its false-acceptance bound is $1-c_{\mathrm{os}}$, an upper bound on the probability that the human executes it incorrectly. Local bounds multiply, so $R_n=\prod_{i=1}^{n}\bar q_i$ can fall below $\delta$ even when each local check is weaker.

On the true-acceptance side, decomposition pays the two costs of Section~\ref{sec:effort}: every local check must succeed, and each receives only part of the budget. Under~\eqref{eq:ci_hardness}, let
\[
c_{\mathrm{os}}:=f(e_{\mathrm{os}},\xi)e^{-\ell_{\mathrm{os}}/r(\xi)},
\qquad \Delta:=\ell_{\mathrm{os}}-\sum_{i=1}^{n}\ell_i.
\]
For positive $f$-terms, the lower bounds on executing every check correctly compare as
\begin{equation}
\begin{aligned}
\frac{\prod_{i=1}^{n}c_i}{c_{\mathrm{os}}}
&=\frac{\prod_{i=1}^{n}f(e_i,\xi)}{f(e_{\mathrm{os}},\xi)}
e^{\Delta/r(\xi)},\\
\prod_{i=1}^{n}c_i>c_{\mathrm{os}}
&\quad\Longleftrightarrow\quad
\Delta>r(\xi)\log\frac{f(e_{\mathrm{os}},\xi)}{\prod_{i=1}^{n}f(e_i,\xi)}.
\end{aligned}
\label{eq:execution_load_comparison}
\end{equation}
With equal allocation, $e_i=E/n$ and $e_{\mathrm{os}}=E$, this threshold is nonnegative: splitting the budget does not raise any $f$-term, so the local loads must together fall below the one-shot load by at least this margin. The sign and size of $\Delta$ must be justified from the cognitive demands of the two procedures. This comparison uses the total load and concerns only true acceptance. Proposition~\ref{prop:separation} instead bounds the load of each local check, which covers both targets.

\begin{proposition}[Conditional certification advantage]
\label{prop:separation}
Fix $(\delta,\gamma)\in(0,1)^2$. Let a local review with $2\le n\le T$ and a one-shot check both meet the common budgets and satisfy the conditions of this section: the local review with bounds $(a_i,c_i)$ and $q$, and the one-shot check with separately justified bounds $a_{\mathrm{os}},c_{\mathrm{os}}\in(0,1]$ and $q_{\mathrm{os}}\in[0,1)$ with $c_{\mathrm{os}}(1-q_{\mathrm{os}})<1$. The one-shot check accepts exactly when it passes; let $R_{\mathrm{os}}:=1-c_{\mathrm{os}}(1-q_{\mathrm{os}})$ and $C_{\mathrm{os}}:=a_{\mathrm{os}}c_{\mathrm{os}}$.
\begin{enumerate}
\item \emph{Local review.} The fixed-length and early-stopping local protocols are $(\delta,\gamma)$-certified if
\begin{equation}
R_n\le\delta,\qquad C_n\ge1-\gamma.
\label{eq:interactive_cert_conditions}
\end{equation}
If $a_i\ge a>0$, $\ell_i\le\ell_0$, and $e_1=\cdots=e_n$, a sufficient condition is
\begin{equation}
\ell_0\le r(\xi)\log\frac{f(e_1,\xi)}{\tau_n},
\qquad
\tau_n:=\max\left\{
\frac{1-\delta^{1/n}}{1-q},\;
\frac{(1-\gamma)^{1/n}}{a}\right\}.
\label{eq:local_reliability_threshold}
\end{equation}
\item \emph{One-shot check.} It is $(\delta,\gamma)$-certified if $R_{\mathrm{os}}\le\delta$ and $C_{\mathrm{os}}\ge1-\gamma$. Under the execution model, $R_{\mathrm{os}}\le\delta$ holds exactly when
\begin{equation}
\ell_{\mathrm{os}}\le r(\xi)\log\frac{(1-q_{\mathrm{os}})f(e_{\mathrm{os}},\xi)}{1-\delta},
\label{eq:one_shot_hardness_threshold}
\end{equation}
which is possible only if $(1-q_{\mathrm{os}})f(e_{\mathrm{os}},\xi)\ge1-\delta$.
\item \emph{Advantage.} If~\eqref{eq:interactive_cert_conditions} holds while $R_{\mathrm{os}}>\delta$ or $C_{\mathrm{os}}<1-\gamma$, the bounds certify the local review, and the one-shot bounds miss at least one target.
\end{enumerate}
\end{proposition}

\begin{proof}
Proposition~\ref{prop:fallible_execution_bounds} bounds false acceptance of the fixed-length protocol by $R_n$ and its rejection of a true claim by $1-C_n$. If~\eqref{eq:interactive_cert_conditions} holds, then $K_\delta\le n$ and $C_{K_\delta}\ge C_n$, so the same proposition certifies the early-stopping protocol. For the sufficient condition, let $\underline c_n:=f(e_1,\xi)e^{-\ell_0/r(\xi)}\le c_i$. Then
\[
R_n\le[1-\underline c_n(1-q)]^n,
\qquad C_n\ge(a\underline c_n)^n,
\]
and both right-hand sides meet their targets exactly when $\underline c_n\ge\tau_n$, which is~\eqref{eq:local_reliability_threshold}. The one-shot check accepts exactly when its single check passes, so Proposition~\ref{prop:fallible_execution_bounds} with $n=1$ gives $R_{\mathrm{os}}$ and $C_{\mathrm{os}}$. Substituting the execution model into $c_{\mathrm{os}}(1-q_{\mathrm{os}})\ge1-\delta$ gives~\eqref{eq:one_shot_hardness_threshold}; since $\ell_{\mathrm{os}}\ge0$, its right-hand side must be nonnegative. No step assumes independence across rounds.
\end{proof}

\paragraph{When decomposition wins.}
With equal allocation and $e_{\mathrm{os}}=E$, writing both conditions as bounds on load, the bounds certify the local review while the one-shot check misses the soundness target (the completeness side is analogous) when
\[
\ell_0\le r(\xi)\log\frac{f(E/n,\xi)}{\tau_n}
\qquad\text{and}\qquad
\ell_{\mathrm{os}}>r(\xi)\log\frac{(1-q_{\mathrm{os}})f(E,\xi)}{1-\delta}.
\]
The load of each local check must therefore fall far enough below the one-shot load. Splitting the budget gives each check less effort, which never helps decomposition. The required per-check reliability changes from $(1-\delta)/(1-q_{\mathrm{os}})$ to $\tau_n$. Its soundness term relaxes with $n$ as evidence accumulates; its completeness term tightens because a true claim must survive every round, and local checks may then need to be \emph{more} reliable than the one-shot check. Both thresholds also scale with $r(\xi)$: greater expertise raises the one-shot threshold and can remove the advantage.

\paragraph{Comparison with an exact one-shot check.}
Table~\ref{tab:global_local_comparison} uses the matrix-product example of Appendix~\ref{sec:comparison_reproduction}, with budget $E=120$ for both designs. The one-shot check multiplies the matrices directly ($q_{\mathrm{os}}=0$); the local review runs five Freivalds checks with $e_i=24$ each ($q=1/2$, $a_i=0.99$). The one-shot check is exact, but the human executes it correctly only with probability $0.938$, so $R_{\mathrm{os}}\approx0.062>\delta=0.05$. Each local check must be more reliable ($\tau_5\approx0.989$ versus $0.95$), yet its load of at most $0.014$ is below the limit $0.020$, while the one-shot load $0.20$ exceeds its limit $0.157$.

\begin{table}[H]
\centering
\small
\begin{tabular}{lrrl}
\hline
Design & False-accept bound & True-accept bound & Certified? \\
\hline
One-shot & 0.061875 & 0.938125 & No \\
Local & 0.032602 & 0.911219 & Yes \\
\hline
\end{tabular}
\caption{Conditional comparison at $E=120$, $\delta=0.05$, and $\gamma=0.10$, with common tools and resource caps.}
\label{tab:global_local_comparison}
\end{table}

The comparison depends on the execution conditions. Reducing only the assumed one-shot load from $\ell_{\mathrm{os}}=0.20$ to $0.10$, below its limit $0.157$, gives $R_{\mathrm{os}}=0.033856$ and $C_{\mathrm{os}}=0.966144$, satisfying both criteria. Appendix~\ref{sec:comparison_reproduction} gives the parameters and calculations.
\section{Empirical Simulation}
\label{sec:experiments}

We use LLM-only simulations to examine whether stable acceptance in a challenge--response dialogue selects more accurate answers. In Prover--Verifier Deliberation (PVD), an LLM prover gives an answer and a set of checkable subclaims. A second LLM takes the verifier's role, accepting, rejecting, or challenging the weakest part of the transcript. Both models remain fixed throughout the exchange; the prover can respond to challenges by supplying support or revising its answer.

When the verifier accepts and the prover has never changed its answer, we label the run Accept~+~No~Change (ANC). The supporting argument may still be elaborated or revised. Rejected runs and runs accepted after an answer change are non-ANC. The no-change condition matches the protocol of Section~\ref{sec:protocol}, where the statement under review is fixed and a revised answer starts a new review. ANC therefore corresponds to the protocol's acceptance event for the initial claim. Whether Theorems~\ref{thm:anytime_soundness} and~\ref{thm:conditional_completeness} apply to it depends on the prescribed checks and threshold rule, which the PVD verifier does not implement; unlike the protocol, it may also accept before issuing any challenge. Acceptance after an answer change concerns a different claim. Comparing the correctness of the ANC and non-ANC groups tests whether stable acceptance is a useful signal of answer quality; Section~\ref{sec:pvd_instantiation} defines the ANC gap used below.

GPQA Diamond\footnote{\url{https://github.com/idavidrein/gpqa}} lets us compare this signal across model pairs. We then test whether verifiers reject deliberately wrong defenses and examine ANC on the more difficult Humanity's Last Exam (HLE) tasks.\footnote{\url{https://www.lastexam.ai/}} The GPQA and HLE results draw on \citet{sedoc2026trust}, whose run records are public;\footnote{\url{https://github.com/jsedoc/prover-verifier-deliberation}, commit \texttt{fd64491}.} the two deceptive-prover stress tests were conducted for this paper. These simulations examine dialogue behavior and answer selection, without validating the human-verifier assumptions or the theoretical guarantees.

\subsection{The ANC Gap}
\label{sec:pvd_instantiation}

For a labeled dataset, we compare accuracy within the ANC and non-ANC groups:
\[
\text{Gap}=\Pr[\mathrm{correct}\mid \mathrm{ANC}]
-\Pr[\mathrm{correct}\mid \neg\mathrm{ANC}].
\]
A positive gap means that ANC selects more accurate answers on the evaluated tasks; a zero or negative gap means that it does not. Computing the gap requires correctness labels as well as recorded verdicts and answer changes. It measures the quality of the selected answers, rather than improvement in a given answer through revision.

The gap conditions on the verdict, whereas Theorems~\ref{thm:anytime_soundness} and~\ref{thm:conditional_completeness} bound acceptance conditional on whether the claim is true. Taking the prover's initial answer as the claim, we also report $\widehat{\mathrm{FA}}$ and $\widehat{\mathrm{TA}}$, the fractions of false and of true initial answers that end in ANC. These are the observed counterparts of the quantities bounded by $\delta$ and $1-\gamma$.

\subsection{GPQA Diamond}
\label{sec:results_gpqa}

Table~\ref{tab:gpqa_simulation} reports PVD results for three prover--verifier pairs from \citet{sedoc2026trust}, each evaluated on all 198 GPQA Diamond questions with one run per question. The model pairs were fixed before any ANC gap was computed.


\begin{table}[t]
\centering
\caption{GPQA Diamond PVD simulations ($N=198$ per pair). ANC and non-ANC accuracies are percentages, with subgroup sizes $n$ in parentheses. Gap is their difference in percentage points, with a two-sided 95\% percentile-bootstrap interval from 10{,}000 question-level resamples (seed 42). $\widehat{\mathrm{FA}}$ and $\widehat{\mathrm{TA}}$: false and true initial answers that end in ANC. Flash-Lite denotes Gemini~3.1~Flash-Lite.}
\label{tab:gpqa_simulation}
\small
\setlength{\tabcolsep}{3pt}
\begin{tabular*}{\textwidth}{@{\extracolsep{\fill}}llccccc@{}}
\hline
\textbf{Prover} & \textbf{Verifier} & \shortstack{\textbf{ANC}\\\textbf{Acc. ($n$)}} & \shortstack{\textbf{non-ANC}\\\textbf{Acc. ($n$)}} & \textbf{Gap (95\% CI)} & $\widehat{\mathrm{FA}}$ & $\widehat{\mathrm{TA}}$ \\
\hline
Sonnet 4.6 & Haiku 4.5 & 93.5 (108) & 51.1 (90) & +42.4 [+31.0, +53.5] & 7/47 & 101/151 \\
GPT-5.4 & GPT-5.4-mini & 97.6 (85) & 62.8 (113) & +34.8 [+25.5, +44.3] & 2/45 & 83/153 \\
Gemini 3.1 Pro & Flash-Lite & 97.3 (112) & 90.7 (86) & +6.6 [$-0.1$, +13.7] & 3/8 & 109/190 \\
\hline
\end{tabular*}
\end{table}

The Sonnet/Haiku and GPT pairings show large positive ANC gaps. The Gemini pairing has high accuracy in both subsets and a smaller gap whose interval includes zero. Thus, high ANC accuracy alone does not establish that stable acceptance separates better answers from the remaining runs. Split by the truth of the initial answer, the Haiku verifier accepted 7 of 47 false (14.9\%) and 101 of 151 true (66.9\%) initial answers without change, and GPT-5.4-mini 2 of 45 (4.4\%) and 83 of 153 (54.2\%). The Gemini prover gave only 8 false initial answers, too few for a stable estimate of $\widehat{\mathrm{FA}}$. Each pair changes both prover and verifier, so differences between rows mix the two roles.

\subsection{Deceptive-Prover Stress Test}
\label{sec:deceptive_soundness}

Soundness asks what happens when the statement under defense is false. We run two GPQA Diamond stress tests in which the prover receives the correct answer but is instructed to defend a wrong option. In the na\"ive condition, the prover selects a plausible wrong answer and argues for it. In the strategic condition, it is additionally told that the verifier is a smaller model and that it should maintain a convincing defense.

Both conditions use Sonnet~4.6 as prover and an unchanged Haiku~4.5 verifier that is unaware of the manipulation. Each condition includes all 198 questions with round budget $T=12$. Table~\ref{tab:deceptive_soundness} reports the outcomes; in neither condition did any run switch to the correct answer.

\begin{table}[t]
\centering
\caption{Deceptive-prover stress tests conducted for this paper on GPQA Diamond ($N=198$ runs per condition, round budget $T=12$). \textbf{Fool rate}: fraction of runs in which the verifier accepted the deliberately wrong answer. Three na\"ive and one strategic runs ended after all three prover-response JSON parsing attempts failed; the operational protocol counted these runs as rejections.}
\label{tab:deceptive_soundness}
\small
\begin{tabular*}{\textwidth}{@{\extracolsep{\fill}}lcc@{}}
\hline
\textbf{Condition} & \textbf{Fool rate} & \textbf{Rejected} \\
\hline
Na\"ive deceptive prover & 12/198 (6.1\%) & 186/198 (93.9\%) \\
Strategic deceptive prover & 18/198 (9.1\%) & 180/198 (90.9\%) \\
\hline
\end{tabular*}
\end{table}

The verifier rejects about 94\% of na\"ive and 91\% of strategic defenses. Every accepted run retains the original wrong answer, so all 30 accepts are ANC runs with wrong final answers. The strategic fool rate is 3.0 percentage points higher, a difference an exact paired McNemar test does not distinguish from noise ($p=0.307$). These rates cover two specified deceptive strategies rather than every adaptive prover $P^*$, so they do not estimate the false-pass bound \(q\) of Section~\ref{sec:fallibility}.

\subsection{HLE}
\label{sec:results_hle}

The HLE runs use the same 513 multiple-choice questions for all three prover--verifier pairs. Table~\ref{tab:hle_boundary_simulation} reports overall accuracy and, where answer-change records are available, the ANC gap. The Sonnet/Haiku point estimate is negative, although its interval includes zero; the GPT/Gemini pairing has a positive gap. Split by the truth of the initial answer, Sonnet/Haiku accepted false initial answers unchanged slightly more often than true ones (32.1\% versus 25.0\%), so its acceptance did not favor true claims; GPT/Gemini accepted true initial answers far more often (73.5\% versus 36.9\%).

\begin{table}[t]
\centering
\caption{PVD simulations on the same HLE multiple-choice subset ($N=513$ per pair). \textbf{Acc.}: overall accuracy. \textbf{Gap}, $\widehat{\mathrm{FA}}$, and $\widehat{\mathrm{TA}}$ are computed as in Table~\ref{tab:gpqa_simulation}. The Opus/Sonnet records omit answer changes, so these three quantities cannot be computed for that pair.}
\label{tab:hle_boundary_simulation}
\small
\begin{tabular*}{\textwidth}{@{\extracolsep{\fill}}llcccc@{}}
\hline
\textbf{Prover} & \textbf{Verifier} & \textbf{Acc.} & \textbf{Gap (95\% CI)} & $\widehat{\mathrm{FA}}$ & $\widehat{\mathrm{TA}}$ \\
\hline
Sonnet 4.6 & Haiku 4.5 & 20.1\% & $-7.1$ [$-13.9$, +0.1] & 134/417 & 24/96 \\
Opus 4.6 & Sonnet 4.6 & 40.0\% & --- & --- & --- \\
GPT-5.5 & Gemini 3.1 Pro & 45.6\% & +27.9 [+19.6, +36.3] & 110/298 & 158/215 \\
\hline
\end{tabular*}
\end{table}

For the Sonnet/Haiku pairing, stable acceptance did not select more accurate answers in this run. One possible explanation is that the verifier lacks the knowledge or tools to check difficult claims and challenges only those it partly understands. In the model of Section~\ref{sec:fallibility}, such limitations can mean a lower execution bound $c_i$ or a higher ideal false-pass bound $q$, leaving no certifiable review length. The observed gap does not identify either bound.

\subsection{Summary}
\label{sec:experiments_summary}

On GPQA, answers accepted without revision are more accurate for two of the three tested pairs. Most deliberately wrong answers in the stress tests are rejected, though some are accepted unchanged. On HLE, where the tasks are harder to check, the Sonnet/Haiku point estimate turns negative. The split by the truth of the initial answer shows the same pattern: on GPQA the Haiku and GPT-5.4-mini verifiers rarely accepted false initial answers unchanged, whereas on HLE the Haiku verifier accepted false initial answers at least as often as true ones. The GPT/Gemini pairing, by contrast, kept a positive gap on HLE. Explaining this variation requires recording which challenges were selected, which evidence was supplied, and which checks were performed; Section~\ref{sec:human_verification_evaluation} proposes such a study for human review.
\section{Discussion}
\label{sec:discussion}

The guarantees in this paper rest on checks that a person can carry out and on bounds for how those checks can fail. The rest of this section asks when dividing a review into rounds is worth its cost, how to design checks that a bounded verifier can both execute and learn from, and how these questions could be tested with human reviewers.

\subsection{When Interaction Helps}

A user may be able to check a particular inference, source record, or calculation even when producing the full argument is beyond their resources. Interaction lets the user request the supporting material needed for that check and use the result to decide what to examine next. Its potential value lies in making the work of verification manageable: selecting a useful question, locating the relevant evidence, and carrying out the check must all fit within the user's budget.

Proposition~\ref{prop:separation} states when this division pays off: the load of each local check must fall far enough to offset the effort spread across rounds and the additional chances to fail. When it does, the supplied bounds can certify the local review while an exact but demanding one-shot check stays uncertified. The proposition concerns decomposition, since it compares specified local and one-shot implementations of the same claim. Whether dialogue adds anything over a static presentation of the same local evidence is a separate question that the analysis does not settle.

The early-stopping protocol accepts at the first threshold crossing, avoiding further checks once the evidence requirement has been met. If reaching that threshold requires more reliable work than the budget supports, a useful response is to improve the checks, their presentation, or the available tools. Simply extending the dialogue need not produce a certifiable review. Reject, in turn, follows a failed check, an inadmissible message, or a round budget that runs out before the threshold is reached; in each case the claim was not accepted, not shown to be false. An interface may report which reason applied while keeping the formal verdict.

\subsection{Designing Checks for Bounded Verifiers}
\label{sec:challenge_design}

In the scheduling example, the organizer selects a requirement that bears on the claim, such as that Lee's two sessions do not overlap, and compares the cited entries against the fixed records. Questions about session topics may aid understanding, but they leave the requirement unchecked. If the selected requirement could have exposed a missed conflict, incomplete records indicate an inadequate response; a mistaken comparison despite complete records indicates an execution error.

A check must do two things at once: have a known chance of exposing an error, and be executable by the person who performs it. A check that covers the whole claim exposes more errors but may exceed the reviewer's attention or expertise; a local check is easier to carry out but exposes only the errors it covers. The complete list of requirements resolves this tension in the example: every conflicting schedule violates at least one listed requirement, so random selection supplies coverage without reviewing every session. This is why the false-pass bound $q$ can come from the structure of the task, while the response and execution bounds $a_i$ and $c_i$ are the harder inputs. These two bounds must hold at every reachable history. Displaying the records together or providing a comparison tool raises the chance of correct execution, and with it the evidence earned per passed check.

For later rounds, the organizer can combine follow-ups prompted by the dialogue with fresh random checks from the complete list. The construction in Appendix~\ref{app:anchored_validity} reserves a chosen probability each round for a reference check with a justified detection bound under correct execution, leaving the remaining questions free to follow the dialogue, and derives the evidence each passed check earns.

That evidence fixes how many passed checks are needed before acceptance. The matrix-product check of Appendix~\ref{app:worked_certificates} reaches the threshold within its budget. A schedule with 100 requirements does not fit a budget of 20 rounds: with $\delta=0.05$ and correct execution probability $0.9$, even a review that draws a random requirement in every round needs more than 300 passed checks.

\subsection{Evaluating Human Verification}
\label{sec:human_verification_evaluation}

The variation in ANC across the simulations motivates examining what verifiers actually check before accepting. The accepted wrong defenses in the stress tests make it useful to distinguish an answer that survives an evidence check from one that merely remains unchanged through an exchange of plausible arguments. A human study could record the selected challenges, the evidence supplied, the checks performed, and the reason for stopping, alongside the final verdict and the truth of the claim.

To separate the benefit of decomposition from the benefit of interaction, such a study could compare three conditions: a specified one-shot review, local review of fixed indexed evidence, and local review with evidence supplied on request. The conditions should use matched claims and source records, common tools, and the same total budget, including the time spent choosing questions and finding evidence; the two local conditions should draw on the same evidence pool and checks, with participants assigned at random within expertise groups. The first comparison isolates decomposition, the second the value of evidence supplied on request.

The central measurement is how false acceptance and rejection of true claims trade off as more rounds are allowed within a fixed budget, since a lower false-acceptance rate may come at the cost of rejecting more true claims. Records of individual checks would explain that tradeoff: if longer reviews reject more true claims, independent assessment of the responses and checks can show whether usable evidence was missing or participants erred despite having it. Comparing where these failures occur across task difficulties and expertise groups would indicate whether smaller checks reduce the burden of execution enough to offset the additional exchanges.
\section{Limitations and Scope}
\label{sec:limitations}

The framework applies most directly to tasks with a usable truth condition and checks that a bounded verifier can perform. Its guarantees concern the claim as formalized. Checking a formal object does not show that it faithfully represents the original natural-language claim, and checking a retrieved source says nothing about whether the source is correct. Applying the framework therefore also requires justifying these correspondences \citep{min2023factscore,jiang2023draft,yang2023leandojo}. The guarantees also attach to a single review of one fixed claim: a revised answer starts a new review, and the theorems do not control error across repeated cycles of revision and review. Tasks without such a truth condition, such as preference elicitation, creative work, and contested judgments, and tasks whose evidence is unavailable or that concern how a program behaves on every input, may benefit from dialogue without admitting the certification guarantees studied here.

The theorems combine supplied false-pass, response, and execution bounds; they do not provide a method for obtaining them from free-form dialogue. Soundness needs no assumption about the prover's honesty, but completeness does. It requires a response bound $a_m$ that the deployed LLM actually meets, and the protocol does not elicit latent knowledge or solve alignment \citep{christiano2021elk,friedl2026impossibility}; without a justified $a_m$, acceptance of a true claim has no useful lower bound. The execution bound must hold at every reachable history against every admissible adaptive prover, including one whose responses are accurate but presented so that the human check is harder to carry out correctly. A success rate averaged across participants, rounds, or tasks is not evidence for such a bound, and we do not analyze updating the bounds during a review. The certification criteria are sufficient conditions for the specified protocols; a review that fails them may still meet its error targets in practice.

Whether people can reliably perform the proposed checks remains untested. Applying the response and execution models to human--LLM review requires evidence about how effort, expertise, and cognitive load affect whether usable evidence is supplied and whether checks are executed correctly. In the numerical examples, the functional form of these relationships is a modeling choice and the parameter values are illustrative. Even in the matrix-product example, where exact algebra supplies the false-pass bound, the response and execution bounds remain assumptions. Counting checks or arithmetic operations does not establish how long human review takes or whether local review reduces cognitive load. Nor do the LLM-only simulations resolve this gap. The PVD verifier runs no prescribed checks and no threshold rule, so its observed acceptance rates are not instances of $\delta$, and the runs reproduce neither human attention, responsibility, nor fatigue, so they cannot identify the human reliability parameters. Section~\ref{sec:human_verification_evaluation} proposes studies to measure these aspects of human review.

The comparison in Proposition~\ref{prop:separation} holds the claim, source records, tools, and budgets fixed and compares specified one-shot and local implementations. It establishes a difference in what the supplied bounds certify, not an ordering of actual error rates or an advantage over all one-shot review methods. Any benefit attributed to dialogue must also be distinguished from the benefit of making local evidence available: the matrix-product check, for example, can be applied to a fixed submitted result without any exchange.

Protecting the prover's private information requires a condition beyond honesty: the full verifier view must be efficiently simulatable without access to the prover's private state. Appendix~\ref{app:hvzk} shows that the prescribed stopping rule preserves this property when it already holds for the continuation interaction; it does not construct such a simulator for LLM dialogue, and in the scheduling and matrix-product examples the property holds only vacuously, because an exact honest prover reveals nothing the verifier could not compute itself. The result covers an honest verifier and an honest prover on the stated true-instance family, including rejection when the round budget runs out, and nothing outside these assumptions.
\section{Conclusion}
\label{sec:conclusion}

Human--LLM deliberation raises a verification problem when the user cannot reproduce the argument behind the LLM's claim. We give conditions under which a human with a limited budget can decide whether to accept the claim by questioning the LLM and checking the details it supplies. The guarantees rest on these checks and their failure bounds; the LLM's internal state is never inspected. Soundness is anytime-valid against every adaptive prover and needs no assumption that the LLM is honest. The completeness bound depends on how reliably the LLM responds and the human executes each check, and on whether enough evidence can accumulate within the round budget.

Adding checks lowers both the false-acceptance and the true-acceptance bound, so the soundness target sets a shortest review length and the completeness target a longest; we characterize the lengths that meet both. Under a common budget, local checks can meet both targets while a one-shot check of the whole claim misses one, even if the one-shot check has the lower false-pass bound under correct execution. This requires each local check to be easy enough to carry out to offset the thinner effort per check and the additional chances to fail.

Applying these results requires checks that fit the verifier's resources, together with justified false-pass, response, and execution bounds. In the LLM-only simulations, whether stable acceptance selected more accurate answers varied with the task and the model pair; the response and execution bounds for human reviewers remain to be measured.

\needspace{6\baselineskip}
\acks{The authors received no funding for this work.}

\raggedbottom
\appendix
\ifextended
\section{Protocol Mapping and Notation}
\label{app:reference_tables}

Table~\ref{tab:protocol_mapping} relates the framework's components and guarantees to those of classical interactive proofs. Table~\ref{tab:notation} collects the notation used in Sections~\ref{sec:protocol}--\ref{sec:fallibility} and Appendices~\ref{app:hvzk}--\ref{app:worked_certificates}.

\begingroup
\footnotesize
\setlength{\LTpre}{0.35\baselineskip}
\setlength{\LTpost}{0.35\baselineskip}
\begin{longtable}{@{}>{\raggedright\arraybackslash}p{0.20\textwidth}>{\raggedright\arraybackslash}p{0.27\textwidth}>{\raggedright\arraybackslash}p{0.47\textwidth}@{}}
\caption{Protocol objects and guarantees. The general human--LLM framework becomes a standard interactive proof only under the conditions in Corollary~\ref{thm:conditional_ip}.}
\label{tab:protocol_mapping}\\
\hline
\textbf{Classical object} & \textbf{Here} & \textbf{Role or condition} \\
\hline
\endfirsthead
\multicolumn{3}{c}{\tablename\ \thetable\ (continued)}\\
\hline
\textbf{Classical object} & \textbf{Here} & \textbf{Role or condition} \\
\hline
\endhead
\hline
\endfoot
Common and auxiliary inputs & $x=\langle s,i_{\mathrm{pub}}\rangle$, $z=i_{\mathrm{priv}}$, and $L$ & The statement $s$ and public task information form the common input; verifier-only information is auxiliary input and does not define the language. \\[0.4em]
Prover & LLM strategy $P$ with private state $\theta$ & Produces history-dependent responses through $Q$; soundness holds against every adaptive admissible strategy. \\[0.4em]
Verifier & Human strategy $V$ & Uses the challenge policy $\pi_V$, the decision function $U$, a message-length limit, and the round budget $T$. \\[0.4em]
Challenge--response transcript & Histories $h_m$; output $(\Omega,d^{(m')})$ & The transcript records the initial response, the verifier's challenges, and the prover's responses; $\Omega$ records the verdict. \\[0.4em]
Honest-prover correctness & Bounds $a_m$ & An adequate response occurs with probability at least $a_m$ at every reachable analysis history; $a_m=1$ is the ideal case. \\[0.4em]
Verifier validity & Bounds $(q_m,c_m)$ & Correct execution occurs with probability at least $c_m$; under correct execution, a false claim passes with probability at most $q_m$, averaged over the challenge policy and the prover's response, giving $\bar q_m=1-c_m(1-q_m)$. \\[0.4em]
Completeness & Theorem~\ref{thm:conditional_completeness} & For a true statement and a response-reliable honest prover, acceptance has probability at least $\prod_{m=1}^{K}a_mc_m$ under the stated verifier and progress conditions, with $K\le T$. \\[0.4em]
Soundness & Theorem~\ref{thm:anytime_soundness} & For a false statement and every adaptive admissible prover, the probability of reaching the acceptance threshold at any round $m\le T$ is at most $\delta$, under the stated verifier conditions. \\[0.4em]
Interactive-proof status & Corollary~\ref{thm:conditional_ip} & Requires a promise problem, a uniform probabilistic polynomial-time verifier, polynomially bounded communication and round budget, and a gap $\prod_{m=1}^{K}a_mc_m-\delta\ge1/p(\kappa)$. \\[0.4em]
HVZK preservation & Corollary~\ref{prop:conditional_zk} & Assumes a uniform simulator for the full view of the continuation verifier and an efficient map $F$ from the full view to the stopped view; preserves computational HVZK on the covered true instances. \\
\end{longtable}
\endgroup

\small
\setlength{\LTpre}{0pt}
\setlength{\LTpost}{0pt}
\begin{longtable}{@{}>{\raggedright\arraybackslash}p{0.25\textwidth}>{\raggedright\arraybackslash}p{0.68\textwidth}@{}}
\caption{Core notation for the protocol, the threshold rule, the effective region and effort model, and the privacy result.}
\label{tab:notation}\\
\hline
\textbf{Symbol} & \textbf{Meaning} \\
\hline
\endfirsthead
\multicolumn{2}{c}{\tablename\ \thetable\ (continued)}\\
\hline
\textbf{Symbol} & \textbf{Meaning} \\
\hline
\endhead
\hline
\endfoot
$\Sigma$, $X=\Sigma^*$, $\mathcal H$, $\parallel$ & Finite alphabet, token space, well-formed encoded histories, and the append operation on histories. \\[0.4em]
$\kappa$, $(\mathcal Y,\mathcal N)$ & Input size $|s|+\operatorname{size}(i_{\mathrm{pub}})$ and the promise problem (true and false instances) in Corollary~\ref{thm:conditional_ip}. \\[0.4em]
$S$, $v$, $L$ & Statement space, total truth map, and language $L=\{s:v(s)=\top\}$. \\[0.4em]
$P$, $P^{\mathrm H}$, $P^*$ & A prover, the honest prover used for completeness, and an arbitrary adaptive prover used for soundness. \\[0.4em]
$\Theta$, $\theta$, $Q$ & Prover private-state space, a private state, and production function $Q:\mathcal H\times\Theta\to X$. \\[0.4em]
$\mathcal I$, $i=(i_{\mathrm{pub}},i_{\mathrm{priv}})$ & Verifier information states, and the public and private parts used in Corollary~\ref{thm:conditional_ip} and Appendix~\ref{app:hvzk}. \\[0.4em]
$\mathcal C$, $C$, $\pi_V$, $U$ & Set of well-formed challenges, admissible challenge set, challenge policy, and decision function. \\[0.4em]
$\upsilon\in X$ & Private verifier record of random bits, check outcomes, and the current evidence value. \\[0.4em]
$T$, $h_m$, $\Pr^{s,P}$, $\Pr_{\Pi}^{s,P}$ & Round budget, history after round $m$, probability over one run on statement $s$ with prover $P$, and the same for protocol $\Pi$. \\[0.4em]
$J$, $s^-$, $s^+$ & Number of requirements in the scheduling example, and a false and a true statement in Section~\ref{sec:fallibility}. \\[0.4em]
$\mathcal F_m$, $\mathcal G_m$ & Observable history after the check of round $m$, and observable history after the response of round $m$ but before its check. \\[0.4em]
$\mathcal J_m$, $\mathcal K_m$ & Analysis counterparts of $\mathcal F_m$ and $\mathcal G_m$ that also record response adequacy and correct execution; $\mathcal K_m=\mathcal J_{m-1}\vee\mathcal G_m$. \\[0.4em]
$\mathcal A_m$, $a_m$ & Event that the LLM's response is adequate, and its lower bound for a true statement given the analysis history. \\[0.4em]
$\mathcal E_m$, $c_m$ & Event that the human executes the check correctly, and its lower bound given the analysis history before the check. \\[0.4em]
$\mathsf{Pass}_m$, $r_m^\circ$, $q_m$, $\bar q_m$ & $\mathsf{Pass}_m$: the round-$m$ check passes. $r_m^\circ$: probability that the prescribed check passes under correct execution. $q_m$: ideal false-pass bound, fixed before the round. $\bar q_m=1-c_m(1-q_m)$: implemented false-pass bound. \\[0.4em]
$\mu_m$, $\nu_m$, $\widetilde\mu_m$ & Challenge distribution in the anchored construction and its reference and adaptive components (Appendix~\ref{app:anchored_validity}). \\[0.4em]
$\lambda_m^{\mathrm{anc}}$, $\rho_m$, $\bar r_m^{\,s,P^*}$ & Probability of using the reference check, its detection lower bound under correct execution, and the pass probability of a given challenge under correct execution. \\[0.4em]
$\delta$, $\gamma$ & Target upper bounds on false acceptance and on failure to accept a true claim. \\[0.4em]
$M_m$ & Evidence value; a test supermartingale under a false statement. \\[0.4em]
$y_s$, $x_m$, $y_m$, $\Omega$, $d^{(m')}$ & Initial response, round-$m$ challenge and response, verdict, and final transcript. \\[0.4em]
$\Pi^{\mathrm{ES}}$, $\Pi^{(n)}$ & Early-stopping protocol, which accepts at the first threshold crossing, and fixed-length protocol, which accepts only after $n$ passed checks. \\[0.4em]
$R_n$, $C_n$, $K_\delta$ & False-acceptance and true-acceptance bounds for $n$ checks, and the first $n$ with $R_n\le\delta$. \\[0.4em]
$K$ & Number of rounds within which an error-free honest run reaches the threshold (Assumption~\ref{assumption:effective_verifier_confidence}, item~\ref{item:progress}). \\[0.4em]
$\mathsf{ER}_{\delta,\gamma}$, $n_{\min}$, $n_{\max}$ & Effective region (Definition~\ref{def:certifiable_effective_region}) and its shortest and longest lengths; $n_{\min}=K_\delta$ when the region is nonempty. \\[0.4em]
$\mathsf{ER}_{\delta,\gamma}(E)$ & Effective region under budget $E$; each length $n$ splits $E$ differently and has its own bounds $a_i^{(n)},c_i^{(n)}$. \\[0.4em]
$E$, $e_i$, $e_{\mathrm{os}}$ & Total review budget, and the effort spent on local round $i$ or on the one-shot check, counting all work for the check. \\[0.4em]
$\xi$, $\ell_i$, $f$, $r$ & Expertise and tool access, cognitive load of round $i$, and the two functions in the execution model~\eqref{eq:ci_hardness}. \\[0.4em]
$\ell_{\mathrm{os}}$, $c_{\mathrm{os}}$, $\Delta$ & Load and execution bound of the one-shot check, and the one-shot load minus the sum of the local loads; the subscript $\mathrm{os}$ marks the one-shot check. \\[0.4em]
$q_{\mathrm{os}}$, $a_{\mathrm{os}}$, $R_{\mathrm{os}}$, $C_{\mathrm{os}}$ & One-shot false-pass and response bounds, justified separately, and the resulting false-acceptance and true-acceptance bounds. \\[0.4em]
$a$, $\ell_0$, $\tau_n$ & Common lower bound on $a_i$, common upper bound on $\ell_i$, and the per-check reliability that suffices for both error targets in a local review of length $n$. \\[0.4em]
$\Theta_H(\kappa,x)$ & Honest private states covered by the HVZK result for instance $x$. \\[0.4em]
$V^{\mathrm{cont}}$, $V^{\mathrm{stop}}$, $\mathsf{View}^{\theta}_{b}$ & Continuation and stopped verifiers, and their full views under $P_\theta$ for $b\in\{\mathrm{cont},\mathrm{stop}\}$. \\[0.4em]
$\mathsf{Sim}_0$, $\mathsf{Sim}_{\mathrm{stop}}$, $F$ & Assumed simulator for the full view, simulator for the stopped view, and the map from the full view to the stopped view, with $\mathsf{Sim}_{\mathrm{stop}}=F\circ\mathsf{Sim}_0$. \\
\end{longtable}
\normalsize
\clearpage
\fi
\section{Conditional Honest-Verifier Zero-Knowledge Preservation}
\label{app:hvzk}
\label{sec:conditional_zk}

Under the threshold rule, when the review stops and how it ends depend on the evidence, so the stopped interaction could in principle reveal something about the prover's private state~$\theta$. The protected knowledge is anything about $\theta$ that the verifier could not compute from its own inputs $(x,z)$, defined below; acceptance still reveals that $x$ is a true instance. The stopped interaction reveals nothing new if the full interaction is already honest-verifier zero knowledge: if a simulator reproduces the full view without access to the prover, applying the stopping rule to its output reproduces the stopped view. The assumptions below make this precise for every outcome, including rejection by a failed check or by the round budget.

Use $x=\langle s,i_{\mathrm{pub}}\rangle$, $z=i_{\mathrm{priv}}$, and the uniformity and polynomial-resource conventions of Corollary~\ref{thm:conditional_ip}. Fix the specified verifier, including how its checks are executed and its protocol parameters. For each covered true instance $x\in\mathcal Y$, prespecify a nonempty family $\Theta_H(\kappa,x)$ of honest private states. The state $\theta$ is fixed during execution; fresh prover randomness is sampled separately.

Write $\mathsf{View}^{\theta}_{b}=\operatorname{View}^{P_\theta}_{V^{b}(z)}(1^\kappa,x)$ for $b\in\{\mathrm{cont},\mathrm{stop}\}$. Full views include the inputs, verifier random tape, initial response $y_s$, all raw visible messages (including any that trigger rejection), visible tool/check records, and the verifier's private record. They cover the entire specified execution distribution, including rejection when the round budget runs out, without conditioning on acceptance. Analysis-only events $\mathcal A_m,\mathcal E_m$ are not additional observations.

The continuation verifier $V^{\mathrm{cont}}$ agrees with the actual verifier $V^{\mathrm{stop}}$ through the stopping point of $V^{\mathrm{stop}}$ in initialization, adaptive challenges, checking, and state updates, including evidence-dependent choices. After $V^{\mathrm{stop}}$ stops, $V^{\mathrm{cont}}$ follows a fixed continuation rule. Let $F$ be a uniform, $\theta$-independent, deterministic polynomial-time map that retains the stopped prefix and computes the evidence value and terminal records using only the view. It is defined on all encoded strings, returning a fixed marker on malformed encodings. With a common polynomial-length random-tape convention, $F$ retains that tape and satisfies
\[
\mathsf{View}^{\theta}_{\mathrm{stop}}
\equiv_d F\!\left(\mathsf{View}^{\theta}_{\mathrm{cont}}\right),
\]
where $\equiv_d$ denotes equality in distribution.

Assume that a single uniform simulator $\mathsf{Sim}_0$ runs in probabilistic polynomial time. It receives only $(1^\kappa,x,z)$ and the fixed public protocol description, with no further advice depending on $\theta$ or access to $P_\theta$, and satisfies
\[
\left\{\mathsf{View}^{\theta}_{\mathrm{cont}}\right\}
\approx_c
\left\{\mathsf{Sim}_0(1^\kappa,x,z)\right\}.
\]
For each probabilistic polynomial-time distinguisher, one negligible advantage bound applies uniformly over all covered $x\in\mathcal Y$, $\theta\in\Theta_H(\kappa,x)$, and allowed polynomial-length $z$. The guarantee is relative to $(x,z)$; $z$ need not be independent of $\theta$.

\begin{corollary}[Conditional HVZK preservation]
\label{prop:conditional_zk}
Under the assumptions above, on the covered family the stopped interaction is computationally honest-verifier zero knowledge, with simulator
\[
\mathsf{Sim}_{\mathrm{stop}}(1^\kappa,x,z)
=F\!\left(\mathsf{Sim}_0(1^\kappa,x,z)\right).
\]
\end{corollary}

\begin{proof}
The composed simulator is uniform probabilistic polynomial time. Let $D$ be a probabilistic polynomial-time distinguisher for stopped views. Then $D\circ F$ is also probabilistic polynomial time, and the equality in distribution gives
\begin{align*}
&\bigl|\Pr[D(\mathsf{View}^{\theta}_{\mathrm{stop}})=1]-\Pr[D(\mathsf{Sim}_{\mathrm{stop}})=1]\bigr|\\
&\qquad=\bigl|\Pr[(D\circ F)(\mathsf{View}^{\theta}_{\mathrm{cont}})=1]-\Pr[(D\circ F)(\mathsf{Sim}_0)=1]\bigr|,
\end{align*}
with all arguments $(1^\kappa,x,z)$ suppressed. The right-hand side is bounded by the negligible function that the simulation assumption gives for $D\circ F$, uniformly over covered $x$, $\theta$, and $z$.
\end{proof}

In the scheduling review of Section~\ref{sec:running_example} and the matrix-product check of Appendix~\ref{sec:freivalds_example}, if the honest prover always answers exactly, its replies are computable from $(x,z)$ and the verifier's random tape, so $\mathsf{Sim}_0$ can run the specified verifier itself and simulation is perfect; the privacy is vacuous because the prover holds nothing the verifier could not compute. With fallible honest provers, the pattern of inadequate responses may itself depend on $\theta$, and a single simulator must reproduce it across $\Theta_H(\kappa,x)$. We do not construct $\mathsf{Sim}_0$ for LLM dialogue in which replies depend on the prover's private state.
\section{The Anchored Challenge Construction}
\label{app:anchored_validity}

This appendix proves the \hyperref[sec:anchored_review_example]{anchored construction} of Section~\ref{sec:anytime_soundness_sub}. In round $m$, the verifier draws its challenge from
\[
\mu_m(\cdot\mid h,i)
=\lambda_m^{\mathrm{anc}}\nu_m(\cdot\mid h,i)
+(1-\lambda_m^{\mathrm{anc}})\widetilde\mu_m(\cdot\mid h,i),
\]
where $\lambda_m^{\mathrm{anc}}\in[0,1]$ is $\mathcal F_{m-1}$-measurable, $\nu_m$ is an admissible reference distribution, and $\widetilde\mu_m$ may adapt arbitrarily to the transcript. Let $\bar r_m^{\,s,P^*}(h,\chi)$ be the probability that the check of challenge $\chi$ passes under correct execution after prover $P^*$ responds at history $h$.

\begin{proposition}[Anchored local validity]
\label{prop:anchored_local_validity}
If an $\mathcal F_{m-1}$-measurable $\rho_m\in[0,1]$ satisfies $\int(1-\bar r_m^{\,s,P^*}(h,\chi))\,\nu_m(d\chi\mid h,i)\ge\rho_m$ for every false statement, adaptive admissible prover, and reachable active history, then item~\ref{item:validity} of Assumption~\ref{assumption:effective_verifier_confidence} holds with $q_m=1-\lambda_m^{\mathrm{anc}}\rho_m$, and $\bar q_m=1-c_m\lambda_m^{\mathrm{anc}}\rho_m$ whenever $0<c_m\lambda_m^{\mathrm{anc}}\rho_m<1$.
\end{proposition}

\begin{proof}
The prover sees only $\chi$, not the component that produced it, so its response is the same under both. Bounding the pass probability by one on the adaptive component gives $\mathbb E^{s,P^*}[r_m^\circ\mid\mathcal F_{m-1}]=\int\bar r_m^{\,s,P^*}\,d\mu_m\le\lambda_m^{\mathrm{anc}}(1-\rho_m)+1-\lambda_m^{\mathrm{anc}}=q_m$.
\end{proof}

Every passed evidence round therefore earns the same evidence, whichever component produced the challenge; crediting the adaptive component with more would require a separate false-pass bound for it.
\section{Worked Certification Examples}
\label{app:worked_certificates}

This appendix applies the criteria of Section~\ref{sec:fallibility} to a matrix-product check\ifextended{} and a scheduling review\fi{} and gives the parameters behind Figure~\ref{fig:effort_budget} and Table~\ref{tab:global_local_comparison}.

\subsection{Certifying a Matrix-Product Check}
\label{sec:freivalds_example}
Fix a prime $p$, matrices $\mathbf A,\mathbf B\in\mathbb F_p^{d\times d}$, and an LLM-produced candidate $\mathbf C\in\mathbb F_p^{d\times d}$, all fixed before the review begins; the statement is $s_{\mathrm{MM}}:\mathbf A\mathbf B=\mathbf C$ over $\mathbb F_p$. In evidence round $m$, the verifier samples a fresh uniform $t_m\in\{0,1\}^d$ and asks for $u_m=\mathbf B t_m$, $v_m=\mathbf A u_m$, and $w_m=\mathbf C t_m$. An exact-arithmetic checker recomputes all three and passes only if the replies agree with it and $v_m=w_m$; the replies show which product is wrong but can be inadequate, so $a<1$ is possible.

This is Freivalds' check \citep{freivalds1977probabilistic}. If $\mathbf D:=\mathbf A\mathbf B-\mathbf C$ has a nonzero entry $D_{jk}$, then, conditional on the other coordinates of $t_m$, $(\mathbf D t_m)_j=0$ for at most one of $t_{m,k}\in\{0,1\}$. Since the matrices are fixed and $t_m$ is fresh, $q=1/2$ after every adaptive history; inconsistent replies only lower the pass probability, and exact honest replies always pass on a true claim. If dialogue-chosen vectors replace the fresh one with probability $1-\lambda^{\mathrm{anc}}$, Appendix~\ref{app:anchored_validity} gives $q=1-\lambda^{\mathrm{anc}}/2$ instead.

Section~\ref{sec:comparison_reproduction} evaluates this check under the effort model of Section~\ref{sec:effort} with budget $E=120$: five checks are certified, four leave $R_4>\delta$, and six or more leave $C_n<1-\gamma$, so the effective region is $\{5\}$.

\ifextended
\subsection{A Scheduling Review beyond the Budget}
\label{sec:anchor_budget}
Apply the \hyperref[sec:anchored_review_example]{anchored construction} of Section~\ref{sec:zkp} to a fixed ten-session schedule. One availability check per session covers its speakers and room; two checks for each of the 45 session pairs test room and speaker conflicts. Both branches use this complete list of $J=100$ requirements and the same checker; the random branch samples uniformly with replacement.

Fix $\lambda^{\mathrm{anc}}=0.2$, $c=0.9$, $\delta=0.05$, and $\gamma=0.2$. Then $q=0.998$ and $\bar q=0.9982$. With every scheduled round an evidence round, the evidence value after $n$ passed checks and the required number of rounds are
\[
M_n=\bar q^{-n},\qquad
\left\lceil\frac{\log(1/\delta)}{-\log\bar q}\right\rceil
=1663.
\]
Thus $K_\delta=1663$ if $T\ge1663$ and is undefined otherwise. At $T=20$, even 20 passes cannot reach the threshold; the early-stopping protocol always rejects by the end of the round budget. Its false-acceptance and true-acceptance probabilities are then both zero.

For fixed-length protocols with $a=1$, the first length meeting the false-acceptance target has $C_{1663}=0.9^{1663}\approx8.04\times10^{-77}<1-\gamma$. Since $C_n$ decreases with $n$, no length satisfies both product criteria.

The 1663-round requirement reflects how little guaranteed evidence each pass contributes. The random branch is used with probability $0.2$, and its uniform draw detects at least one violated requirement among 100 with probability at least $1/100$ under correct execution. Including the execution bound gives a detection lower bound of $0.9\times0.2/100=0.0018$, so each pass adds only $-\log(0.9982)\approx0.001802$ toward the threshold $\log20$. The evidence process credits no additional detection from the dialogue-based branch. Anchoring is not the main cost: with $\lambda^{\mathrm{anc}}=1$, so that every round draws a random requirement, $\bar q=0.991$ and the threshold still needs $332$ passed checks. The bottleneck is the detection probability $1/J$ of a single draw. With ideal responses and execution, checking all 100 requirements once would instead settle the fixed claim. Such an exhaustive audit uses a different acceptance rule and needs its own execution-error analysis.
\fi

\subsection{Parameters for \texorpdfstring{Figure~\ref{fig:effort_budget} and Table~\ref{tab:global_local_comparison}}{Figure 3 and Table 1}}
\label{app:fig3_reproducibility}
\phantomsection
\label{sec:comparison_reproduction}
Both use the effort model~\eqref{eq:ci_hardness} with $\xi=10$ and natural logarithms:
\[
\begin{aligned}
r(\xi)&=1+\log(1+\xi),
&b(\xi)&=0.91+0.035\,\frac{\log(1+\xi)}{1+\log(1+\xi)},\\
f(e,\xi)&=\min\bigl\{0.995,\ b(\xi)+0.10\bigl(1-\exp(-0.04e)\bigr)\bigr\},
&\ell_i&=0.010+0.001(i-1),
\end{aligned}
\]
so $r(10)\approx3.3979$ and $b(10)\approx0.9347$. Figure~\ref{fig:effort_budget} evaluates~\eqref{eq:effort_effective_region} at $E=20,25,\ldots,180$ and $n\le T=14$ with $c_i^{(n)}=f(E/n,10)\exp(-\ell_i/r(10))$, $q=0.20$, $\delta=0.10$, $\gamma=0.20$, and $a_i^{(n)}=1$; for example, $\mathsf{ER}_{\delta,\gamma}(20)=\{2,3,4\}$ and $\mathsf{ER}_{\delta,\gamma}(180)=\{2,\ldots,10\}$. R scripts for Figures~\ref{fig:effective_region} and~\ref{fig:effort_budget} are available from the authors.

Table~\ref{tab:global_local_comparison} uses fixed $20\times20$ matrices over $\mathbb F_{101}$, $E=120$, $\delta=0.05$, and $\gamma=0.10$. The one-shot check multiplies the matrices with $e_{\mathrm{os}}=120$, $q_{\mathrm{os}}=0$, $a_{\mathrm{os}}=1$, and assumed load $\ell_{\mathrm{os}}=0.20$; the local review runs five Freivalds checks with $e_i=24$, $q=1/2$, and $a_i=0.99$. Both use the same exact-arithmetic tool and fit a $20{,}000$-operation and $4{,}096$-word cap ($15{,}600$ versus $11{,}700$ field multiplications and additions). Since $f(24,10)=f(120,10)=0.995$,
\[
\begin{aligned}
R_5&=\prod_{i=1}^{5}(1-c_i/2)\approx0.032602,
& C_5&=\prod_{i=1}^{5}0.99\,c_i\approx0.911219,\\
R_{\mathrm{os}}&=1-c_{\mathrm{os}}\approx0.061875,
& C_{\mathrm{os}}&=c_{\mathrm{os}}\approx0.938125.
\end{aligned}
\]
The local review first reaches the threshold at round five ($R_4\approx0.064617>\delta$).

\vskip 0.2in
\bibliography{ref}

\end{document}